\documentclass[accepted]{uai2023} 

\usepackage[american]{babel}

\usepackage{natbib} 
    \usepackage{url}
\usepackage{mathtools} 
\usepackage{booktabs} 
\usepackage{tikz} 

\usepackage[utf8]{inputenc} 
\usepackage[T1]{fontenc}    
\usepackage{hyperref}       
\usepackage{url}            
\usepackage{booktabs}       
\usepackage{amsfonts}       
\usepackage{nicefrac}       
\usepackage{microtype}      

\usepackage{mathtools} 
\usepackage{booktabs} 
\usepackage{tikz} 

\usepackage{tablefootnote}
\usepackage{enumitem}

\usepackage{graphicx}
\usepackage{amsthm}
\usepackage{amsmath}
\usepackage{amssymb}
\usepackage{mathtools}
\usepackage{ulem}

\usepackage{float}

\usepackage{caption}
\usepackage{subcaption}  
\DeclareMathOperator*{\E}{\mathbb{E}}

\usepackage{color}

\newcommand{\s}{\mathcal{S}}
\newcommand{\A}{\mathcal{A}}

\newcommand{\T}{\mathcal{T}}
\newcommand{\TL}{\mathcal{F}}
\newcommand{\Q}{\mathcal{Q}}

\theoremstyle{plain}
\newtheorem{theorem}{Theorem}[section]

\newtheorem{lemma}[theorem]{Lemma}

\theoremstyle{definition}
\newtheorem{definition}[theorem]{Definition}

\theoremstyle{remark}
\newtheorem{remark}[theorem]{Remark}

\theoremstyle{plain}
\newtheorem{innercustomgeneric}{\customgenericname}
\providecommand{\customgenericname}{}
\newcommand{\newcustomtheorem}[2]{%
  \newenvironment{#1}[1]
  {%
   \renewcommand\customgenericname{#2}%
   \renewcommand\theinnercustomgeneric{##1}%
   \innercustomgeneric
  }
  {\endinnercustomgeneric}
}

\newcustomtheorem{customthm}{Theorem}
\newcustomtheorem{customlemma}{Lemma}
\newcustomtheorem{customcor}{Corollary}

\usepackage{tablefootnote}
\usepackage{enumitem}

\usepackage{graphicx}
\usepackage{amsthm}
\usepackage{amsmath}
\usepackage{amssymb}
\usepackage{mathtools}
\usepackage{ulem}

\newcommand{\overbar}[1]{\mkern 1.5mu\overline{\mkern-1.5mu#1\mkern-1.5mu}\mkern 1.5mu}

\title{Bounding the Optimal Value Function in Compositional Reinforcement Learning}

\author[1]{\href{mailto:<jacob.adamczyk001@umb.edu>?Subject=Your UAI 2023 paper on Compositional RL}{Jacob~Adamczyk}{}}
\author[2]{Volodymyr~Makarenko}
\author[1]{Argenis~Arriojas}
\author[2]{Stas~Tiomkin}
\author[1]{Rahul~V.~Kulkarni}

\affil[1]{%
Department of Physics\\
University of Massachusetts Boston\\
Boston, MA, USA
}
\affil[2]{%
    Department of Computer Engineering\\
    San Jos\'e State University\\
    San Jos\'e, CA, USA
}

\begin{document}
\maketitle

\newif\ifarx
\arxfalse

\begin{abstract}
In the field of reinforcement learning (RL), agents are often tasked with solving a variety of problems differing only in their reward functions.
In order to quickly obtain solutions to unseen problems with new reward functions, a popular approach involves functional composition of previously solved tasks. 
However, previous work using such functional composition has primarily focused on specific instances of composition functions whose limiting assumptions allow for exact zero-shot composition. 
Our work unifies these examples and provides a more general framework for compositionality in both standard and entropy-regularized RL. 
We find that, for a broad class of functions, the optimal solution for the composite task of interest can be related to the known primitive task solutions. Specifically, we present double-sided inequalities relating the optimal composite value function to the value functions for the primitive tasks. We also show that the regret of using a zero-shot policy can be bounded for this class of functions.  
The derived bounds can be used to develop clipping approaches for reducing uncertainty during training, allowing agents to quickly adapt to new tasks. 
\end{abstract}

\section{Introduction}\label{sec:intro}
Reinforcement learning has seen great success recently, but still suffers from poor sample complexity and task generalization. 
Generalizing and transferring domain knowledge to similar tasks remains a major challenge in the field.
To combat this, different methods of transfer learning have been proposed; such as the option framework \citep{sutton_4room, barreto2019option}, successor features \citep{dayan1993improving, barreto_sf, hunt_diverg, nemecek}, and functional composition \citep{Todorov, Haarnoja2018, peng_MCP, boolean_stoch, vanNiekerk}. In this work, we focus on the latter method of ``compositionality'' for transfer learning.

Research in compositionality has focused on the development of approaches to combine previously learned optimal behaviors to obtain solutions to new tasks. In the process, many instances of functional composition in the literature have required limiting assumptions on the dynamics and allowable class of reward functions (goal-based rewards in \citep{boolean}) in order to derive exact results. Furthermore, previous work has focused on isolated examples of particular functions in either standard or entropy-regularized RL and a framework for studying a general class of composition functions without limiting assumptions is currently lacking. One of the main contributions of our work is to provide a unifying general framework to study compositionality in reinforcement learning.

In our approach, we focus on ``primitive'' tasks which differ only in their associated reward functions. 
More specifically, we consider those downstream tasks whose reward functions can be written as a global function of the known source tasks' reward functions. 
To maintain generality, we do not assume that transition dynamics are deterministic. We also do not assume that reward functions are limited to the goal-based setting, in which there are a limited number of absorbing ``goal'' states \citep{Todorov, vanNiekerk} defining the primitive task. Given the generality of this setting, we cannot expect to obtain exact solutions for compositions as in prior work. Instead, we provide a class of functions which can be used to obtain approximate solutions and bounds on the corresponding downstream tasks.

Given the solutions to a set of primitive tasks, we show that it is possible to leverage such information to obtain approximate solutions for a large class of compositely-defined tasks.
To do so, we relate the solution of the downstream composite task to the solved primitive (source) tasks. Specifically, we derive relations on the optimal value function of interest. From such a relation, a ``zero-shot'' (i.e. not requiring further training) policy can be extracted for use in the composite domain of interest.
We then show that the suboptimality (regret) of this zero-shot policy is upper bounded.

Our results support the idea that RL agents can focus on obtaining domain knowledge for simpler tasks, and later use this knowledge to effectively solve more difficult tasks. The primary contributions of the present work are as follows:

\textbf{Main contributions}
\begin{itemize}
    \item Establishing a general framework for analyzing reward transformations and compositions for the case of stochastic dynamics, globally varying reward structures, and continuing tasks.
    
    \item Derivation of bounds on the respective optimal value functions for transformed and composite tasks.
    
    \item Demonstration of zero-shot approximate solutions and value-based clipping of new tasks based on the known optimal solutions for primitive tasks. 
\end{itemize}

\section{Background}

In this work, we analyze the case of finite, discrete state and action spaces, with the Markov Decision Process (MDP) model \citep{suttonBook}. Let $\Delta(X)$ represent the set of probability distributions over $X$. Then the MDP is represented as a tuple $\T=\langle \s,\A,p, \mu, r,\gamma \rangle$ where $\s$ is the set of available states; $\A$ is the set of possible actions; $p: \s \times \A \to \Delta(\s)$ is a transition function describing the system dynamics; $\mu \in \Delta(\s)$ is the initial state distribution; $r: \s \times \A \to \mathbb{R}$ is a (bounded) reward function which associates a reward (or cost) with each state-action pair; and $\gamma \in (0,1)$ is a discount factor which discounts future rewards and guarantees the convergence of total reward for infinitely long trajectories ($T \to \infty$).

In ``standard'' (un-regularized) RL, the agent maximizes an objective function which is the expected future reward:
\begin{equation}\label{eq:std_rl_obj}
J(\pi) = \underset{\tau \sim{} p, \pi}{\mathbb{E}}
\left[ \sum_{t=0}^{\infty} \gamma^{t} r(s_t,a_t) \right].
\end{equation}

This objective has since been generalized for the setting of  entropy-regularized RL \citep{ZiebartThesis, LevineTutorial}, which augments the standard RL objective in Eq. \eqref{eq:std_rl_obj} by appending an entropic regularization term for the policy:

\begin{equation}
J(\pi)=
\underset{\tau \sim{} p, \pi}{\mathbb{E}}
\left[ \sum_{t=0}^{\infty} \gamma^{t} \left( r(s_t,a_t) - \frac{1}{\beta}
\log\left(\frac{\pi(a_t|s_t)}{\pi_0(a_t|s_t)} \right) \right) \right]
\label{eq:ent_rl_obj}
\end{equation}
where $\pi_0: \s \to \Delta(\A)$ is the fixed prior policy. The inverse temperature parameter, $\beta \in (0, \infty)$, regulates the contribution of entropic costs relative to the accumulated rewards.

The additional entropic control cost discourages the agent from choosing policies that deviate too much from the prior policy. Importantly, entropy-regularized MDPs lead to stochastic optimal policies that are provably robust to perturbations of rewards and dynamics \citep{eysenbach}; making them a more suitable choice for real-world problems.

By ``solution to the RL problem'', we hereon refer to the corresponding \textit{optimal} action-value function $Q(s,a)$ from which an \textit{optimal} control policy can be derived: $\pi(s) \in \text{argmax}_a Q(s,a)$ for standard RL; and $\pi(a|s) \propto \exp(\beta Q(s,a))$ for entropy-regularized RL. Note that these definitions are consistent with the limit $\beta \to \infty$ in which the standard RL objective is recovered from Eq.~\eqref{eq:ent_rl_obj}. For both standard and entropy-regularized RL, the optimal $Q$-function can be obtained by iterating a recursive Bellman equation. For standard RL, the Bellman optimality equation is given by \citep{suttonBook}:

\begin{equation}
    Q(s,a) = r(s,a) + \gamma \mathbb{E}_{s' \sim{} p(\cdot|s,a)} \max_{a'} \left( Q(s',a') \right)
\end{equation}

The entropy term in the objective function for entropy-regularized RL modifies the previous optimality equation to \citep{ZiebartThesis, Haarnoja_SAC}:

\begin{equation}
    Q(s,a) = r(s,a) + \frac{\gamma}{\beta} \E_{s'} \log \mathbb{E}_{a' \sim{} \pi_0(\cdot|s')}  e^{ \beta Q(s',a') }
\end{equation}

One of the primary goals of research in compositionality and transfer learning is deriving results for the optimal $Q$ function for new tasks based on the known optimal $Q$ function(s) for primitive tasks. There exist many forms of composition and transfer learning in RL, as discussed by \cite{taylor_survey}. In this paper, we focus on the case of concurrent skill composition by a single agent as opposed to an options-based approach \citep{sutton_4room}, or other hierarchical compositions \citep{hierarchy, saxe_hier}. We elaborate on this point with the definitions below.

To formalize our problem setup, we adopt the relevant definitions provided by \citep{ourAAAI}:

\begin{definition}
A \textbf{primitive RL task} is specified by an MDP $\mathcal{T} = \langle \s,\A,p,r, \gamma \rangle$ for which the optimal $Q$ function is known.
\end{definition}

In this work, we focus on primitive tasks with general reward functions, i.e. including both goal-based (sparse rewards on absorbing sets such as in the linearly solvable MDP framework of \citep{Todorov, boolean, vanNiekerk}) \textit{and} arbitrary reward landscapes \citep{Haarnoja2018}. 

\begin{definition}
The \textbf{transformation of an RL task} is defined by its (bounded and continuous) transformation function: $f: \mathbb{R} \to \mathbb{R}$ and a primitive task $\mathcal{T}$. The transformed task shares the same states, actions, dynamics, and discount factor as $\mathcal{T}$ but has a transformed reward function $\widetilde{r}(s,a) = f(r(s,a))$.
\end{definition}

\begin{definition}
The \textbf{composition of $M$ RL tasks} is defined by a (bounded and continuous) function $F: \mathbb{R}^M \to \mathbb{R}$ and a set of primitive tasks $\{\mathcal{T}^{(k)}\}$. The \textbf{composite} RL task is defined by a new reward function $\widetilde{r}(s,a)~=~F(\{r^{(k)}(s,a)\})$; and shares the same states, actions, dynamics, and discount factor as all the primitive RL tasks.
\end{definition}

Finally, we define the Transfer Library, the set of functions which obey the hypotheses of our subsequent results (see Sections \ref{sec:std_lemmas} and \ref{sec:entropy-regularized_lemmas}). This definition serves to facilitate the general discussion of results obtained.
\begin{definition}
Given a set of primitive tasks $\{\mathcal{T}^{(k)}\}$, the \textbf{Transfer Library}, denoted by $\TL$, is the set of all transformation (or composition, when $M > 1$) functions $f$ which admit double-sided bounds (see Sections \ref{sec:std_lemmas} and \ref{sec:entropy-regularized_lemmas}) 
on the composite task's optimal $Q$ function ($\widetilde{Q}$).
\end{definition}

Specifically,  $\TL = \{f\ | f\ \textrm{satisfies Lemma 4.1 or 4.3}\}$ for standard RL and  $\TL = \{f\ | f\ \textrm{satisfies Lemma 5.1 or 5.3}\}$ for entropy-regularized RL.

We have empirically found (cf. Fig.~\ref{fig:tasks} and Supplementary Material) that by using the derived bounds for the optimal value function, the agent can learn the optimal policy more efficiently for tasks in the Transfer Library.

\section{Previous Work}\label{sec:prev_work}

There is much previous work concerning compositionality and transfer learning in reinforcement learning. In this section we will give a brief overview by highlighting the work most relevant for the current discussion.

In this work, we focus on value-based composition; rather than policy-based composition \cite{peng_MCP}, features-based composition \citep{barreto_sf}, or hierarchical (e.g. options-based) composition \citep{alver, sutton_4room, barreto2019option}.

Value based methods of composition use the optimal value functions of lower-level or simpler ``primitive'' tasks to derive an approximation (or in some cases exact solution) for the composite task of interest.
In the optimal control framework, \citep{Todorov}
has shown that optimal value functions can be composed exactly for linearly-solvable MDPs with a $\textrm{LogSumExp}$ or ``soft OR'' composition over primitive tasks; assuming that tasks share the same absorbing set (boundary states). With a similar assumption of the shared absorbing set, \citep{boolean} show that exact optimal value functions for Boolean compositions may be recovered from primitive task solutions; thereby allowing an exponential improvement in knowledge acquisition.

In more recent work, in the context of MaxEnt RL, \cite{Haarnoja2018} have shown that linear convex-weighted compositions in stochastic environments result in a bound on optimal value functions, and the policy extracted from this zero-shot bound is indeed useful for solving the composite task. The same premise of convex-weighted reward structures was studied by \cite{hunt_diverg} where the difference between the bound of \citep{Haarnoja2018} and the optimal value function can itself be learned, effectively tightening the bound until convergence. This notion of a corrective function was subsequently generalized by \cite{ourAAAI} to allow for arbitrary functions of composition in entropy-regularized RL.

Other authors have considered the question of linear task decomposition, for instance \citep{barreto_sf} where a convex weight vector over learned \textit{features} can be calculated to solve the transfer problem over linearly-decomposable reward functions in standard RL. More recent developments on this line of research include \citep{hong2022bilinear} where a more general ``bilinear value decomposition'', conditioned on various goals, is learned. In \citep{kimconstrained}, the authors consider the successor features (SFs) framework of \citep{barreto_sf}, and propose lower and upper bounds on the optimal value function of interest. They show that by replacing standard generalized policy improvement (GPI) with a constrained version which respects their bounds, they are able to transfer knowledge more successfully to future tasks in the successor features framework. 

With our reduced assumptions (any constant dynamics, constant discount factor, any rewards) it is not generally possible to solve the transformed or composed tasks based only on primitive knowledge. Nevertheless, we are able to derive bounds on the optimal $Q$-functions in both standard and entropy-regularized RL, from which we can immediately derive policies which fare well in the transformed and composed problem settings. Additionally, we are able to prove that the derived policies have a bounded regret, in a similar form as \cite{Haarnoja2018}'s Theorem 1; but in a more general setting.

\section{Standard RL}\label{sec:std_lemmas}
\subsection{Transformation of Primitive Task}
In this section, we consider \textbf{transformations of a primitive task} in the ``standard'' (un-regularized) RL setting. We assume a solved primitive task is given with reward function $r(s,a)$. Transforming this underlying reward function gives rise to a new reward function, $f(r(s,a))$ which specifies a new RL task to solve. All other variables defining the MDP ($ \s, \A, p, \mu, \gamma $) are assumed to be fixed. In this new setting, we consider how to use the solution to the primitive task (that is, with rewards $r$) to inform the solution of the new, transfer task (that is, with rewards $f(r)$). The set of all applicable functions $f$ for which we can derive bounds, forms the aforementioned \textit{Transfer Library} with respect to the primitive task, for standard RL.

For a general class of transformations of reward functions (as defined below), we show that the optimal value function for the transformed task is bounded by an analogous functional transformation of the optimal value function for the primitive task. (The proofs for all theoretical results are provided in the Supplementary Material.)

We use the following definitions in the subsequent (standard RL) results: Let $X$ be the codomain for the $Q$ function of the primitive task ($Q: \s \times \A \to X \subseteq \mathbb{R}$). Let $V_f$ denote the state-value function derived from the transformation function $f(Q)$: $V_f(s) = \max_a f(Q(s,a))$.
\begin{lemma}[Convex Conditions]\label{thm:convex_cond_std}
Given a primitive task with discount factor $\gamma$ and a bounded, continuous transformation function $f~:~X~\to~\mathbb{R}$ which satisfies:  
\begin{enumerate}
    \item $f$ is convex on its domain $X$\footnote{This condition is not required for deterministic dynamics.\label{dynamics condition}};
    \item $f$ is sublinear:
    \begin{enumerate}[label=(\roman*)]
        \item $f(x+y) \leq f(x) + f(y)$ for all $x,y \in X$
        \item $f(\gamma x) \leq \gamma f(x)$ for all $x \in X$ 
    \end{enumerate}
    \item $f\left( \max_{a} \mathcal{Q}(s,a) \right) \leq \max_{a}~f\left( \mathcal{Q}(s,a) \right)$ for all functions $\mathcal{Q}:~\s~\times~\A \to X.$\footnote{Although this condition is automatically satisfied, it allows for a smoother connection to the analogous hypotheses in Lemmas \ref{thm:concave_cond_std}, \ref{thm:forward_cond_entropy-regularized}, \ref{thm:reverse_cond_entropy-regularized} and compositional results in the Supplementary Material.}
\end{enumerate}

then the optimal action-value function for the transformed rewards, $\widetilde{Q}$, is now related to the optimal action-value function with respect to the original rewards  by:

\begin{equation}\label{eqn:convex_std}
    f(Q(s,a)) \leq \widetilde{Q}(s,a) \leq f(Q(s,a)) + C(s,a)
\end{equation}

where $C(s,a)$ is the optimal value function for a task with reward
\begin{equation}\label{eq:std_convex_C_def}
    r_C(s,a) = f(r(s,a)) + \gamma \E_{s' \sim{} p } V_f(s') - f(Q(s,a)).
\end{equation}
that is, $C$ satisfies the following recursive equation:
\begin{equation}
    C(s,a) = r_C(s,a) + \gamma \E_{s' \sim{} p} \max_{a'} C(s',a').
\end{equation}
\end{lemma}

With this result, we have a double-sided bound on the values of the optimal $Q$-function for the composite task. 
In particular, the lower bound ($f(Q)$) provides a zero-shot approximation for the optimal $Q$-function. It is thus of interest to analyze how well a policy $\pi_f$ extracted from such an estimate ($f(Q)$) might perform.
To this end, we provide the following result which bounds the suboptimality of $\pi_f$ as compared to the optimal policy.

\begin{lemma}
Consider the value of the policy $\pi_f(s) = \max_{a} f(Q(s,a))$ on the transformed task of interest, denoted by $\widetilde{Q}^{\pi_f}(s,a)$.
The sub-optimality of $\pi_f$ is then upper bounded by:
\begin{equation}
    \widetilde{Q}(s,a) - \widetilde{Q}^{\pi_f}(s,a) \leq D(s,a)
\end{equation}
where $D$ is the value of the policy $\pi_f$ in a task with reward
\begin{align*}
     r_D(s,a) = \gamma \E_{s'\sim{} p}\E_{a' \sim{} \pi_f} &\biggr[ \max_{b} \big\{ f(Q(s',b)) + C(s',b) \big\} \\ &- f(Q(s',a')) \biggr]
\end{align*}
that is, $D$ satisfies the following recursive equation:
\begin{equation}
    D(s,a) = r_D(s,a) + \gamma \E_{s'\sim{}p} \E_{a' \sim{} \pi_f} D(s',a').
\end{equation}
\end{lemma}
Interestingly, the previous result shows that for functions $f$ admitting a tight double-sided bound (that is, a relatively small value of $C$), the associated zero-shot policy $\pi_f$ can be expected to perform near-optimally in the composite domain.

Another class of functions for which general bounds can be derived arises when $f$ satisfies the following ``reverse'' conditions.
\begin{lemma}[Concave Conditions]\label{thm:concave_cond_std}
Given a primitive task with discount factor $\gamma$ and a bounded, continuous transformation function $f~:~X~\to~\mathbb{R}$ which satisfies:  
\begin{enumerate}
    \item $f$ is concave on its domain $X$\textsuperscript{\ref{dynamics condition}};
    \item $f$ is superlinear: 
    \begin{enumerate}[label=(\roman*)]
        \item $f(x+y) \geq f(x) + f(y)$ for all $x,y \in X$ 
        \item $f(\gamma x) \geq \gamma f(x)$ for all $x \in X$
    \end{enumerate}
    \item $f\left( \max_{a} \mathcal{Q}(s,a) \right) \geq \max_{a}~f\left( \mathcal{Q}(s,a) \right)$ for all functions $\mathcal{Q}:~\s~\times~\A \to X.$
\end{enumerate}
   
then the optimal action-value functions are now related in the following way:
\begin{equation}\label{eqn:concave_std}
    f(Q(s,a)) - \hat{C}(s,a) \leq \widetilde{Q}(s,a) \leq f(Q(s,a))
\end{equation}

where $\hat{C}$ is the optimal value function for a task with reward 
\begin{equation}
    \hat{r}_C(s,a) = f(Q(s,a)) - f(r(s,a)) - \gamma \E_{s' \sim{} p} V_f(s')
\end{equation}
\end{lemma}
One obvious way to satisfy the final condition in the preceding lemma is to consider functions $f(x)$ which are monotonically increasing. 
Note that the definitions of $C$ and $\hat{C}$ guarantee them to be positive, as is required for the bounds to be meaningful (this statement is shown explicitly in the Supplementary Material).
Furthermore, by again considering the derived policy $\pi_f(a|s)$, we next provide a similar result for concave conditions, noting the difference in definitions between $D$ and $\hat{D}$.
\begin{lemma}
Consider the value of the policy $\pi_f(s) = \max_{a} f(Q(s,a))$ on the transformed task of interest, denoted by $\widetilde{Q}^{\pi_f}(s,a)$.
 The sub-optimality of $\pi_f$ is then upper bounded by:
\begin{equation}
    \widetilde{Q}(s,a) - \widetilde{Q}^{\pi_f}(s,a) \leq \hat{D}(s,a)
\end{equation}
where $\hat{D}$ is the value of the policy $\pi_f$ in a task with reward
\begin{equation*}
     \hat{r}_D = \gamma \underset{s'\sim{} p\ }{\mathbb{E}} \underset{a' \sim{} \pi_f}{\mathbb{E}} \biggr[ V_f(s') - f(Q(s',a')) + \hat{C}(s',a') \biggr]
\end{equation*}
\end{lemma}

\renewcommand{\arraystretch}{1.5} 
\begin{table}[ht]
    \centering
    \begin{tabular}{ll}
        \toprule
        \multicolumn{2}{c}{Standard RL Results}                   \\
        \cmidrule(r){1-2}
        Transformation     &  Result  \\
        \midrule
            Linear Map: & $\widetilde{Q}(s,a) = k Q(s,a)$  \\
            Convex conditions: & $\widetilde{Q}(s,a) \geq f(Q(s,a))$ \\
            Concave conditions: & $\widetilde{Q}(s,a) \leq f(Q(s,a))$  \\
            OR Composition: & $\widetilde{Q}(s,a) \geq \max_k \{Q^{(k)}(s,a)\}$  \\
            AND Composition: & $\widetilde{Q}(s,a) \leq \min_{k} \{ Q^{(k)}(s,a)\}$   \\
            NOT Gate:
            & $ \widetilde{Q}(s,a) \geq   
            - Q(s,a)$  \\
            Conical combination: & $\widetilde{Q}(s,a) \leq \sum_k \alpha_k Q^{(k)}(s,a)$  \\

        \bottomrule
    \end{tabular}
    
        \caption{\textbf{Standard Transfer Library.} Lemmas \ref{thm:convex_cond_std}, \ref{thm:concave_cond_std} stated in Section \ref{sec:std_lemmas} lead to a broad class of applicable transfer functions in standard RL. In this table we list several common examples which are demonstrated throughout the paper and in the Supplementary Materials. We show only one side of the bounds from Eq. \eqref{eqn:convex_std}, \eqref{eqn:concave_std} which requires no additional training.}\label{tab:std_rl}
    
\end{table}

We remark that the conditions imposed on the function $f$ are not very restrictive. For example, the Boolean functions and linear combinations considered in previous work are all included in our framework, while we also include novel transformations not considered in previous work (see Table~\ref{tab:std_rl}). Furthermore, the conditions for $f$ can be further relaxed if specific conditions are met. For the case of deterministic dynamics, the first condition is not required ($f$ need not be convex nor concave). 

We have shown that in the standard RL case, quite general conditions (convexity and sublinearity) lead to a wide class of applicable functions defining the Transfer Library. The conditions given in Lemmas \ref{thm:convex_cond_std} and \ref{thm:concave_cond_std} are straightforward to check for general functions. When given a primitive task defined by a reward function $r$, one can therefore bound the optimal $Q$ function for a general transformation of the rewards, $f(r)$, when $f$ obeys the conditions above. This new set of transformed tasks defines the Transfer Library from a given set of primitive tasks.

The previous (and following) results are presented for the case in which the primitive task $Q$-values are known \textit{exactly}. In practice, however, this is not typically the case, even in tabular settings. In continuous environments where the use of function approximators is necessary, the error that is present in learned $Q$-values is further increased. To address this issue, we provide an extension of all double-sided bounds for the case where an $\varepsilon$-optimal estimate of the primitive task's $Q$-values is known, such that $|Q(s,a)~-~\bar{Q}(s,a)|~\leq~\varepsilon$ for all $s,a$. To derive such an extension, we further require that the composition function $f$ is $L$-Lipschitz continuous (essentially a bounded first derivative), i.e. $|f(x_1)-f(x_2)|\leq L|x_1-x_2|$ for all $x_1,x_2 \in X$, the domain of $f$ (in the present case, the $x_i$ are the primitive task's $Q$-values). To maintain the focus of the main text, we provide these results and the corresponding proofs in the Supplementary Material. We note that all functions listed in Table \ref{tab:std_rl} and \ref{tab:entropy-regularized} are indeed $1$-Lipschitz continuous.

\subsection{Generalization to Composition of Primitive Tasks}\label{sec:comp_std}

The previous lemmas can be extended to the case of multivariable transformations (see Supplementary Material for details), where $X \to \bigotimes X^{(k)}$ (the Cartesian product of primitive codomains). That is, with a function $F: \bigotimes X^{(k)} \to \mathbb{R}$ and a collection of $M$ subtasks, $\{r^{(k)}(s,a)\}_{k=1}^{M}$, one can synthesize a new, \textbf{composition of subtasks}, with reward defined by $r^{(c)}(s,a) = F(r^{(1)}(s,a), \dotsc, r^{(M)}(s,a))$. 

In this vectorized format, $F$ must obey the above conditions in each argument:
\begin{itemize}
    \item $F$ is convex (concave) in each argument,
    \item $F$ is sublinear (superlinear) in each argument.
\end{itemize}
For the final conditions, we also require a similar vectorized inequality, which we spell out in detail in the Supplementary Material.

As an example of composition in standard RL, we consider the possible sums of reward functions, with each task having a positive weight associated to it.

In such a setup, the agent has learned to solve a set of primitive tasks, then it must solve a task with a new compositely-defined reward function, say $f\left(r^{(1)}, \dotsc, r^{(M)}\right)~\doteq~\sum_{k=1}^{M} \alpha_k r^{(k)}$ for (possibly many) target tasks defined by the weights $\{\alpha_k\}$.
To determine which bound is satisfied for such a composition function, we look to the vectorized conditions above. This function is linear in all arguments, so we must only check the final condition. Since the inequality 
 \begin{equation}
        \sum_k \alpha_k  \max_a  Q^{(k)}(s,a) \geq \max_a \sum_k \alpha_k Q^{(k)}(s,a)
  \end{equation}
holds for any set of $\alpha_k > 0$, this function conforms to the concave vectorized conditions, implying that $\widetilde{Q}(s,a)~\leq~f(Q^{(k)}(s,a))=\sum_k~\alpha_k Q^{(k)}(s,a)$. We can then use the right-hand side of this bound to calculate the associated state-value function ($V_f(s) = \max_a f(Q^{(k)}(s,a))$) and the associated greedy policy ($\pi_f(s) = \text{argmax}_a f(Q^{(k)}(s,a))$). This result agrees with an independent result by \cite{nemecek} (the upper bound in Theorem 1 therein) without accounting for approximation errors.

\section{entropy-regularized RL}\label{sec:entropy-regularized_lemmas}
\subsection{Transformation of Primitive Task}
We will now extend the results obtained in the previous section to the case of entropy-regularized RL. 
Again we first consider the single-reward transformation $f(r)$ for some function $f$. Here we state the conditions that must be met by functions $f$, which define the Transfer Library for entropy-regularized RL.

We now use the following definitions in the subsequent (entropy-regularized RL) results.
In the following results, we set $\beta=1$ for brevity, and the expectation in the final condition is understood to be over actions, sampled from the prior policy. Full details can be found in the proofs provided in the Supplementary Material.
\begin{lemma}[Convex Conditions]
\label{thm:forward_cond_entropy-regularized}
Given a primitive task with discount factor $\gamma$ and a bounded, continuous transformation function $f~:~X~\to~\mathbb{R}$ which satisfies:  
\begin{enumerate}
    \item $f$ is convex on its domain $X$\textsuperscript{\ref{dynamics condition}};
    \item $f$ is sublinear:
       \begin{enumerate}[label=(\roman*)]
        \item $f(x+y) \leq f(x) + f(y)$ for all $x,y \in X$
        \item $f(\gamma x) \leq \gamma f(x)$ for all $x \in X$
    \end{enumerate}
    \item $f\left( \log \E \exp \mathcal{Q}(s,a) \right) \leq \log \E \exp f\left( \mathcal{Q}(s,a) \right)$ for all functions $\mathcal{Q}:~\s~\times~\A \to X.$
\end{enumerate} 

then the optimal action-value function for the transformed rewards, $\widetilde{Q}$, is now related to the optimal action-value function with respect to the original rewards by:
\begin{equation}\label{eq:convex_entropy-regularized}
   f \left( Q(s,a) \right) \leq \widetilde{Q}(s,a) \leq f \left( Q(s,a) \right) + C(s,a)
\end{equation}
\end{lemma}

\begin{table}[ht] 
    \centering
    \begin{tabular}{ll}
        \toprule
        \multicolumn{2}{c}{Entropy-Regularized RL Results}  \\
        \cmidrule(r){1-2}
        Transformation& Result\\
        \midrule
            Linear Map, $k \in (0,1)$\tablefootnote{Note that linear reward scaling can also be viewed as a linear scaling in the temperature parameter.}:& $\widetilde{Q}(s,a) \geq k Q(s,a) $ \\
            Linear Map, $k > 1$:& $\widetilde{Q}(s,a) \leq k Q(s,a)$ \\
            Convex conditions:  & $\widetilde{Q} \geq f(Q(s,a))$ \\
            Concave conditions: & $\widetilde{Q} \leq f(Q(s,a))$ \\
            OR Composition: & $\widetilde{Q}(s,a) \geq \max_k \{Q^{(k)}(s,a)\}$       \\
            AND Composition:    & $ \widetilde{Q}(s,a) \leq \min_k \{ Q^{(k)}(s,a)\}$ \\
            NOT Gate: & $\widetilde{Q}(s,a) \geq 
            -Q(s,a)$          \\
            Convex combination\tablefootnote{This extends to the case $\sum_k \alpha_k \geq 1$ by composing with a linear scaling, which respects the same inequality.}:    & $\widetilde{Q}(s,a) \leq \sum_k \alpha_k Q^{(k)}(s,a)$ \\
        \bottomrule
    \end{tabular}
    \caption{\textbf{Entropy-Regularized Transfer Library.} Lemmas \ref{thm:forward_cond_entropy-regularized}, \ref{thm:reverse_cond_entropy-regularized} lead to a broad class of applicable transfer functions in entropy-regularized RL. In this table we list several common examples which are demonstrated throughout the paper and in the Supplementary Materials. We show only one side of the bounds from Eq. \eqref{eq:convex_entropy-regularized}, \eqref{eq:concave_entropy-regularized} which requires no additional training.}\label{tab:entropy-regularized}
\end{table}

We note that $C$ has the same definition as before, but with $V_f$ replaced by its entropy-regularized analog: $V_f(s)~\doteq~\log \E_{a \sim{} \pi_0} \exp f\left(Q(s,a)\right)$.

\begin{lemma}
Consider the soft value of the policy $\pi_f(a|s)~=~\pi_0(a|s)\exp\left( f(Q(s,a)) - V_f(s) \right)$ on the transformed task of interest, denoted by $\widetilde{Q}^{\pi_f}(s,a)$.
 The sub-optimality of $\pi_f$ is then upper bounded by:
\begin{equation}
    \widetilde{Q}(s,a) - \widetilde{Q}^{\pi_f}(s,a) \leq D(s,a)
\end{equation}
where $D$ is the soft value of the policy $\pi_f$
 with reward
\begin{equation*}
    r_D(s,a) = \gamma \E_{s'} 
\left[ \max_{b} \left\{ f\left(Q(s',b)\right) + C(s',b) \right\} -V_f(s') \right].
\end{equation*}
\end{lemma}

Conversely, for concave conditions we have
\begin{lemma}[Concave Conditions]
\label{thm:reverse_cond_entropy-regularized}
Given a primitive task with discount factor $\gamma$ and a bounded, continuous transformation function $f~:~X~\to~\mathbb{R}$ which satisfies:  
\begin{enumerate}
    \item $f$ is concave on its domain $X$\textsuperscript{\ref{dynamics condition}}; 
    \item $f$ is superlinear:
     \begin{enumerate}[label=(\roman*)]
        \item $f(x+y) \geq f(x) + f(y)$ for all $x,y \in X$ 
        \item $f(\gamma x) \geq \gamma f(x)$ for all $x \in X$ 
    \end{enumerate}
    \item $f\left( \log \E \exp \mathcal{Q}(s,a) \right) \geq \log \E \exp f\left( \mathcal{Q}(s,a) \right)$ for all functions $\mathcal{Q}:~\s~\times~\A \to X$.

\end{enumerate}
then the optimal action-value function for the transformed rewards obeys the following inequality:
\begin{equation}\label{eq:concave_entropy-regularized}
    f\left( Q(s,a) \right) - \hat{C}(s,a) \leq \widetilde{Q}(s,a) \leq f \left( Q(s,a) \right)
\end{equation}
\end{lemma}

As in the preceding section, we provide a similar result for the derived policy $\pi_{f}$, given the concave conditions provided. 

\begin{lemma}
Consider the soft value of the policy $\pi_f(a|s)$ on the transformed task of interest, denoted by $\widetilde{Q}^{\pi_f}(s,a)$.
 The sub-optimality of $\pi_f$ is then upper bounded by:
\begin{equation}
    \widetilde{Q}(s,a) - \widetilde{Q}^{\pi_f}(s,a) \leq \hat{D}(s,a)
\end{equation}
where $\hat{D}$ satisfies the following recursive equation
\begin{equation}
    \hat{D}(s,a) = \gamma \E_{s' \sim{} p}\E_{a' \sim{} \pi_f} \left( \hat{C}(s',a') + \hat{D}(s',a') \right).
\end{equation}
\label{lem:concave_regret_maxent}
\end{lemma}

Now, by taking $V_f(s)$ as the previously defined \textit{soft} value function, the fixed points $C$ and $\hat{C}$ have the same definitions as presented in Lemma \ref{thm:convex_cond_std} and \ref{thm:concave_cond_std}, respectively with this new definition of $V_f$.

This final constraint (in Lemma 5.1 and 5.3) on $f$ arises out of the requirements for extending the previous results to entropy-regularized RL.
Although the final condition (similar to a log-convexity) appears somewhat cumbersome, we show that it is nevertheless possible to satisfy it for several non-trivial functions (Table~\ref{tab:entropy-regularized}). For instance, functions defining Boolean composition over subtasks ($\max(\cdot), \ \min(\cdot)$), which have not been considered in previous entropy-regularized results \citep{Haarnoja2018, vanNiekerk} as well as new functional transformations such as the NOT gate (Table~\ref{tab:entropy-regularized}). 

\begin{figure*}[ht]
\begin{minipage}[b]{0.4\textwidth}
  \includegraphics[width=\linewidth]{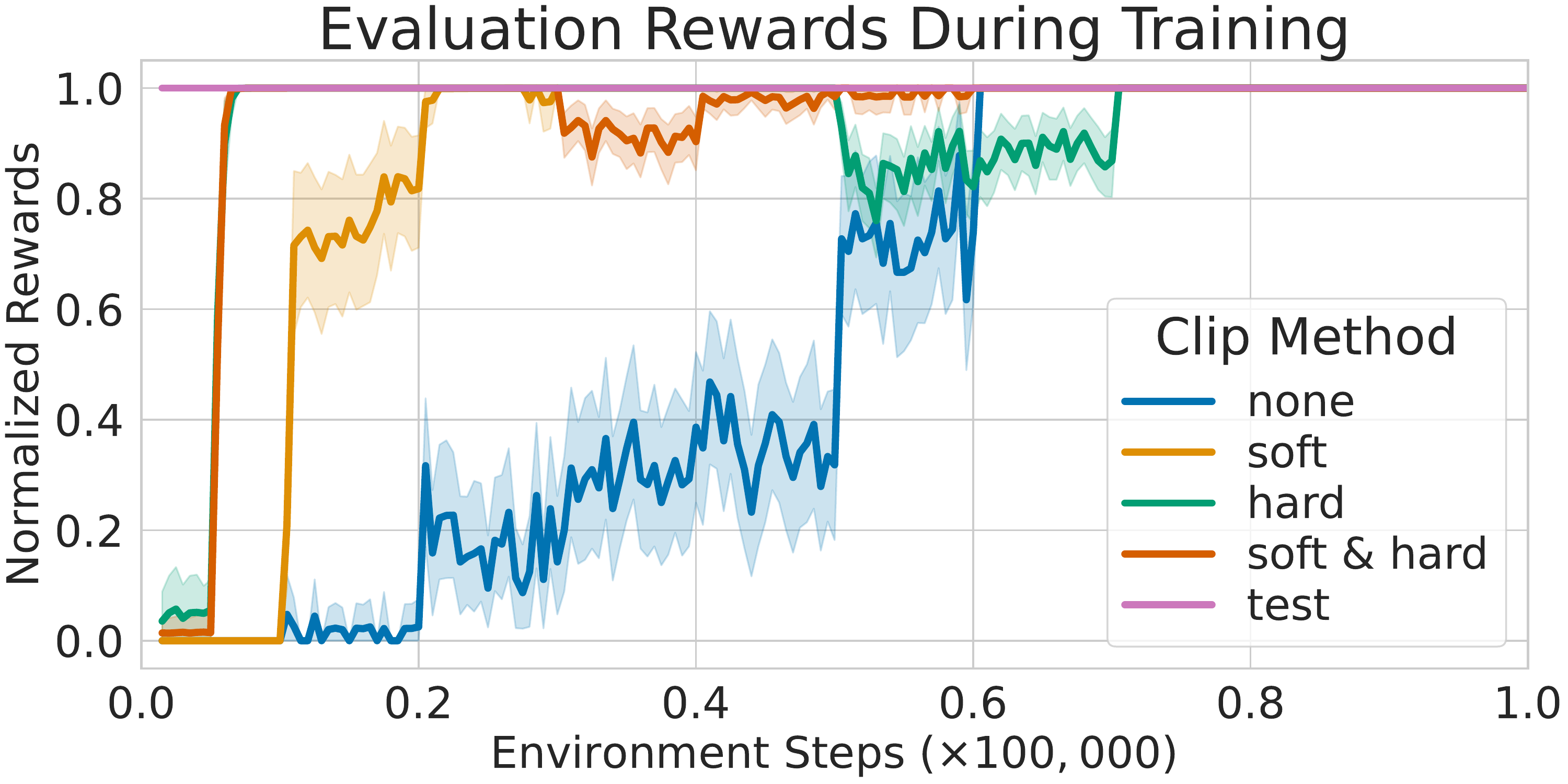}
\end{minipage}
\hspace{2em}
\begin{minipage}[b]{0.175\textwidth}
  \includegraphics[width=\linewidth]{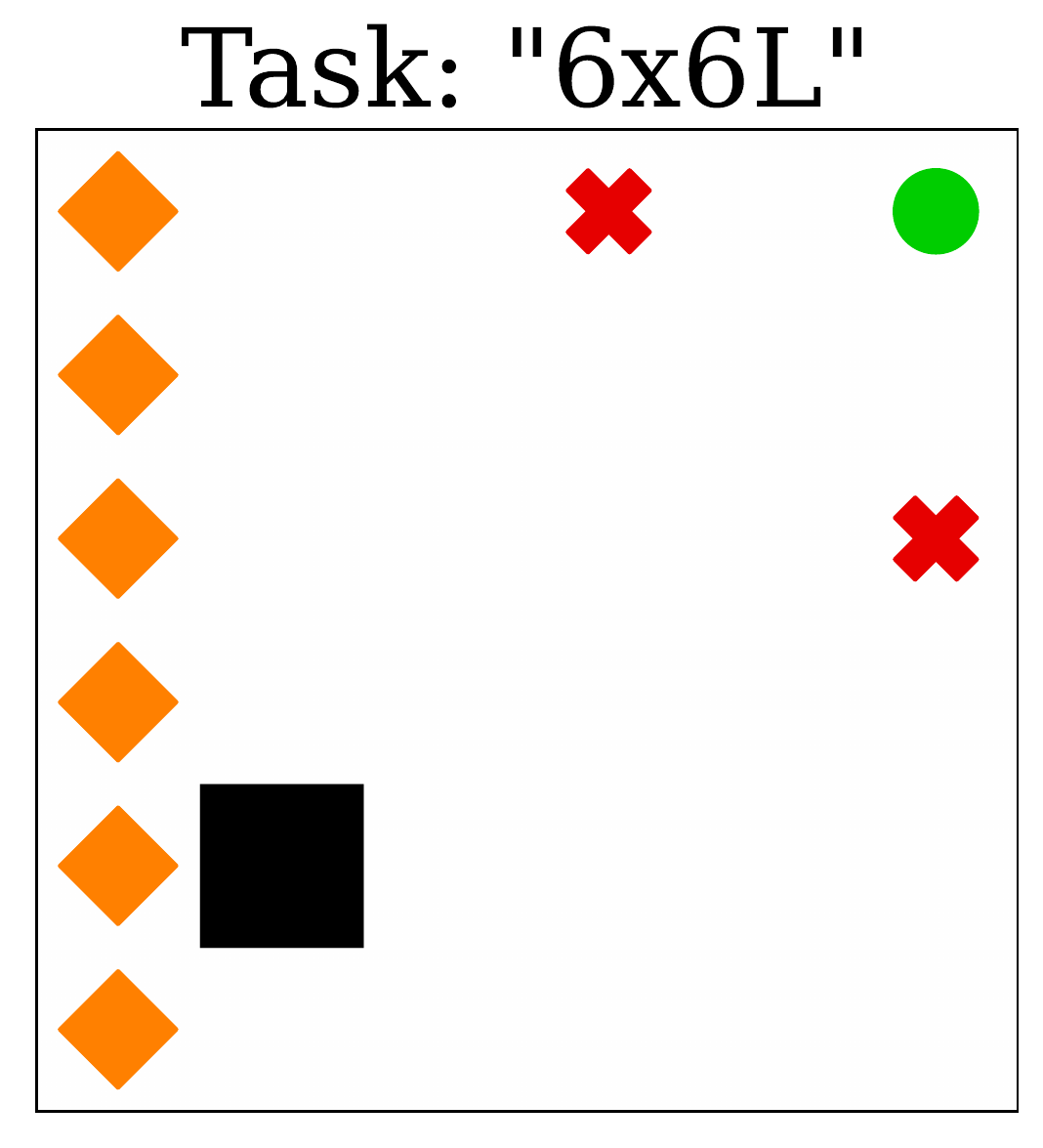}

\end{minipage}
\hspace{0.2em} 
\begin{minipage}[b]{0.175\textwidth}
  \includegraphics[width=\linewidth]{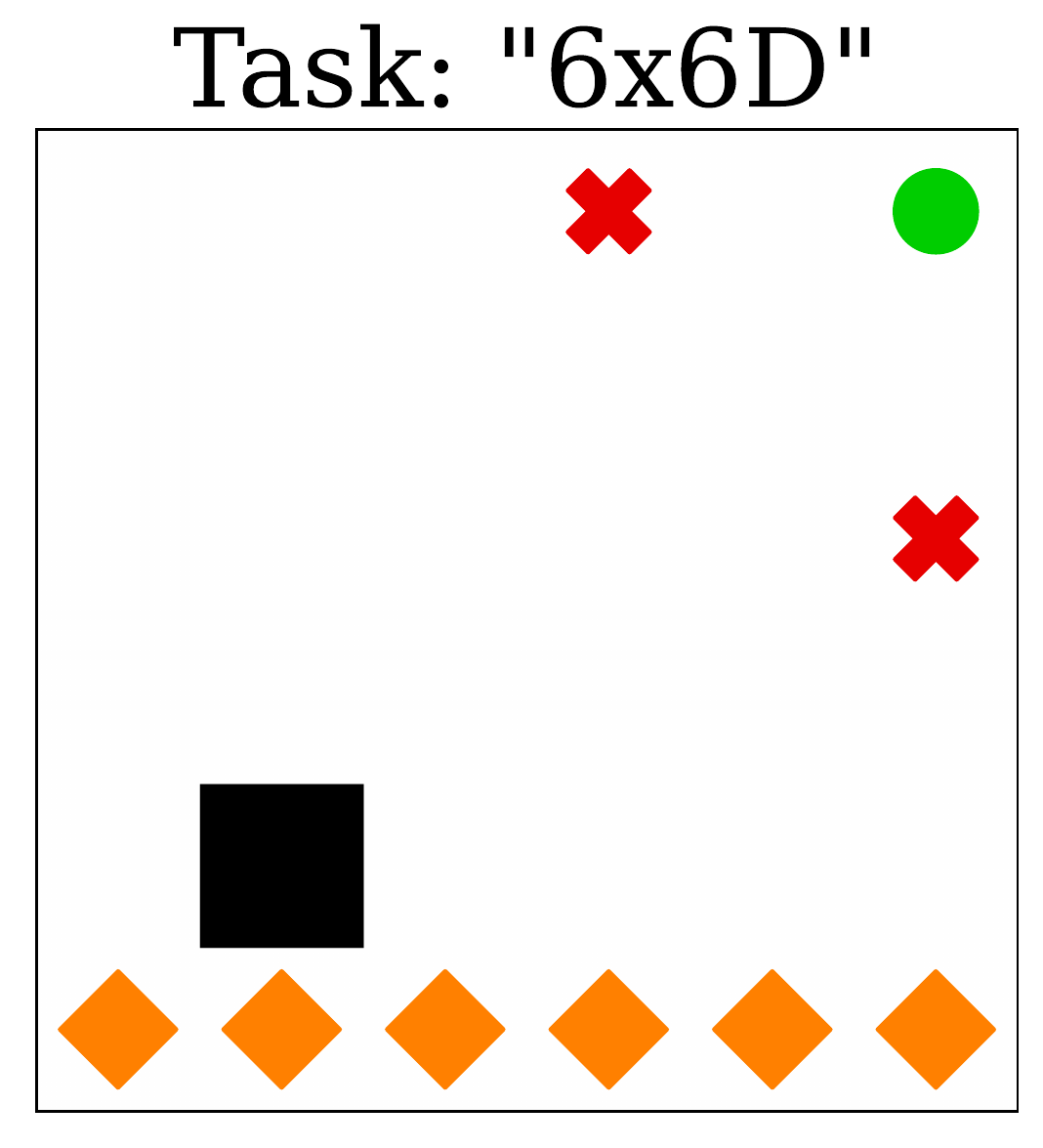}

\end{minipage}
\hspace{0.2em}
\begin{minipage}[b]{0.175\textwidth}
  \includegraphics[width=\linewidth]{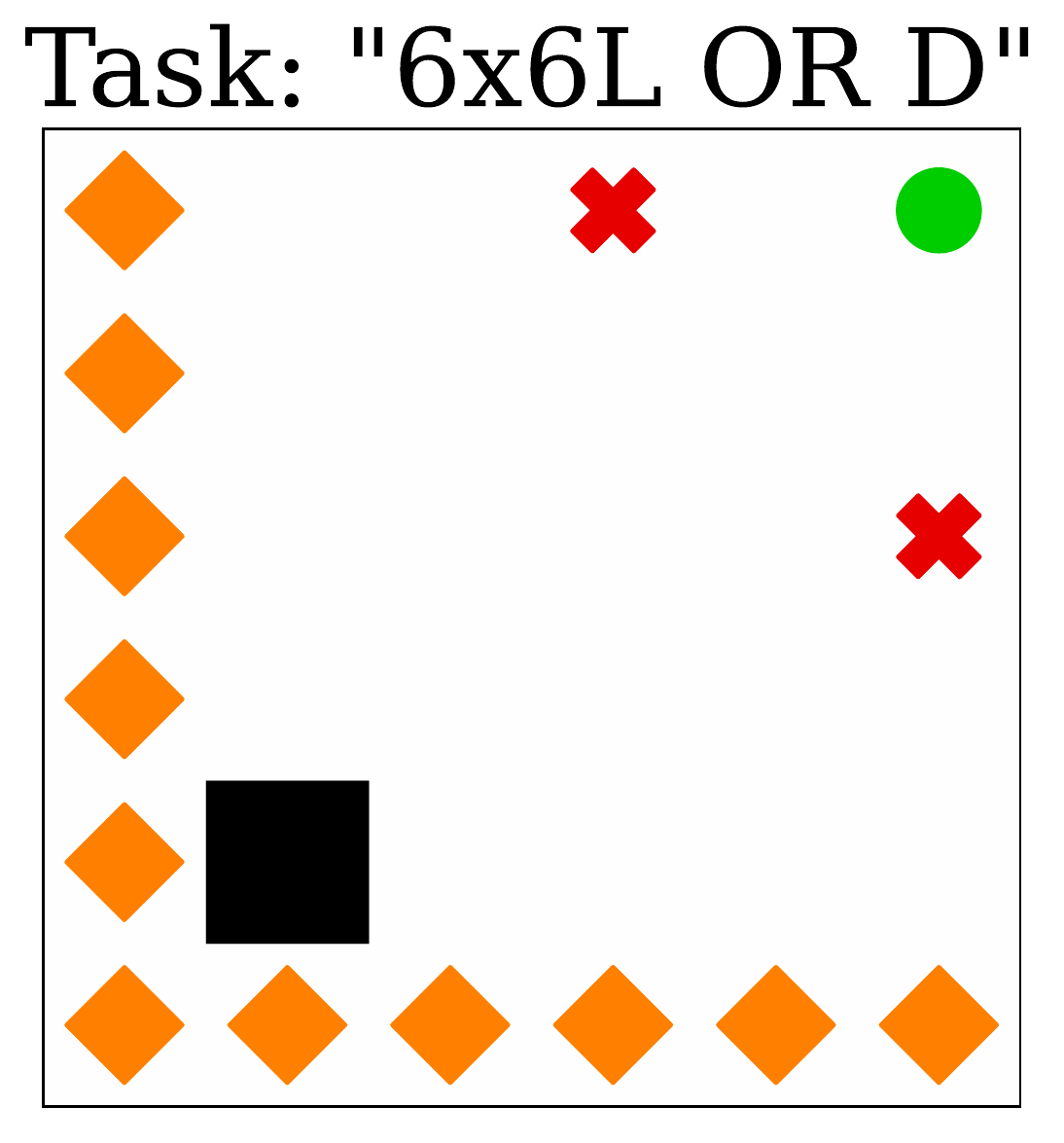}

\end{minipage}

\caption{In the first panel, we show learning curves for each of the clipping methods proposed, averaged over $50$ trials, with a 95\% confidence interval shown in the shaded region. In the next two panels, we depict the primitive tasks with rewarding states (orange diamonds) on the left side and bottom side of the maze, respectively. In the rightmost panel, we show the composite task of interest, with the multivariable ``OR'' composition function $\max_k\{...\}$ used to define its reward function. The agent first solves the two primitive tasks with a deep $Q$-network (DQN, as implemented by Stable-Baselines3 \citep{stable-baselines3}) with results shown in the first panel. Training hyperparameters are given in the Supplementary Material.}
\label{fig:tasks}
\end{figure*}
\subsection{Generalization to Composition of Primitive Tasks}

As we have done in the standard RL setting (Section \ref{sec:comp_std}), we can also extend the previous results to include compositionality: functions operating over multiple primitive tasks.

In this case, \cite{Haarnoja2018} have demonstrated a special case of Lemma \ref{thm:reverse_cond_entropy-regularized} for the composition function $f(\{r^{(k)}\}) = \sum_k \alpha_k r^{(k)}$ for convex weights $\alpha_k$. This can also be shown in our framework by proving the final condition of Lemma \ref{thm:reverse_cond_entropy-regularized} (since the others are automatic given that $f$ is linear). This vectorized condition can be proven via H\"older's inequality.

Besides this previously studied composition function, we can now readily derive value function bounds for other transformations and compositions, for example Boolean compositions as defined previously. The corresponding results for entropy-regularized RL are summarized in Table \ref{tab:entropy-regularized}.

\section{Experiments}

To test our theoretical results using function approximators (FAs), we consider a deterministic ``gridworld'' MDP amenable to task composition\footnote{Source code available at \url{https://github.com/JacobHA/Q-Bounding-in-Compositional-RL}}. Figure ~\ref{fig:tasks} shows the environments of the trained primitive tasks ``$6\times6\ \text{L}$'' and ``$6\times6\ \text{D}$'', whose reward functions are then combined to produce a composite task, ``$6\times6\ \text{L OR D}$''. The agent has $4$ possible actions (in each of the cardinal directions) and begins at the green circle in all cases. The agent's goal is to navigate to the orange states which provide a reward. We note that these states are \textit{not absorbing} unlike the cases considered in prior work. The red ``X'' indicates a penalizing state where the agent's episode is immediately terminated. Finally, a wall (black square) is added for the agent to navigate around. The primitive tasks are assumed to be solved with high accuracy (i.e. we assume the $Q$-values for primitive tasks to be exactly known). Although the domain is rather simple, we use such an experiment as a means of validating our theoretical results while gaining insight on the experimental effects of \textit{clipping} (discussed below) during training.

\begin{figure}[ht!]
    \centering
    \includegraphics[width=0.45\textwidth]{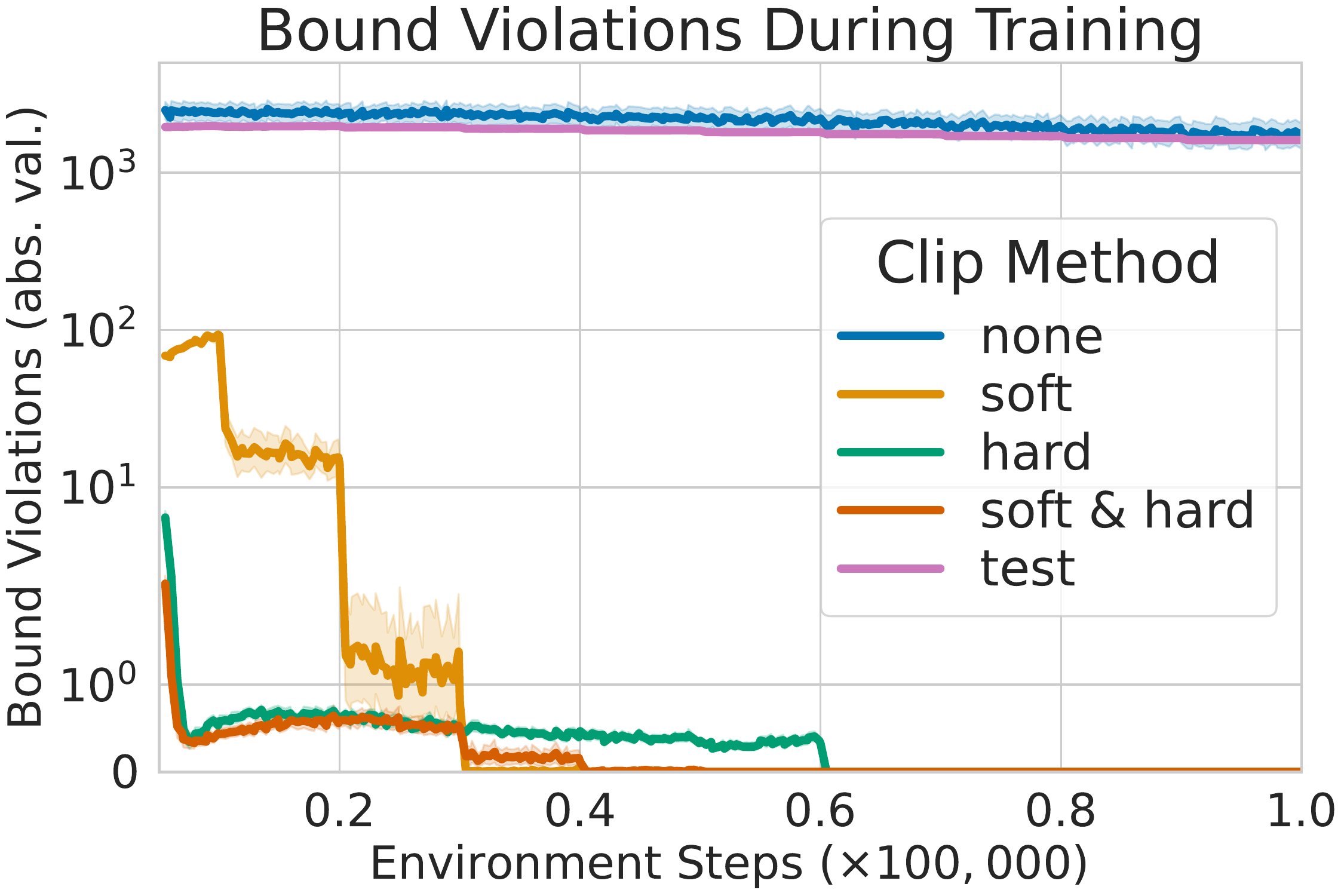}
    \caption{Mean bound violation, shown with shaded 95\% confidence intervals. The bound violation measures the difference between the Q-network's estimate and the allowed bound $\widetilde{Q} \geq f(Q)$ for a given batch during training. Note that the $x$-axis corresponds to zero bound violation (symlog $y$-scale). Although ``test'' clipping does very well in terms of its evaluation performance, it does not respect the bounds, even long after its apparent convergence.}
    \label{fig:bound_viol}
\end{figure}

With the primitive tasks solved, we now consider training on a target composite (``OR'') task. We learn from scratch (with no prior information or bounds being applied) as our baseline (blue line, denoted ``none'', in Fig.~\ref{fig:tasks} and Fig.~\ref{fig:bound_viol}).

To implement the derived bounds, we consider the one-sided bound, thereby not requiring further training. In this case (standard RL, ``OR'' composition) we have the following lower bound (see Table~\ref{tab:std_rl}): $Q^{(\text{OR})} \ge \max\{Q^{(\text{L})}, Q^{(\text{D})}\}$. There are many ways to implement such a bound in practice. One na\"ive method is to simply clip the target network's new (proposed target) value to be within the allowed region for each of the $\widetilde{Q}(s,a)$ that are currently being updated. We term this method ``hard clipping''. Inspired by Section 3.2 of \citep{kimconstrained}, we can also use an additional penalty by adding to the loss function the absolute value of bound violations that occur (the quantity ``$\textrm{BV}$'' defined in Eq.~\eqref{eq:bound_violations}). We term this method as ``soft clipping''. As mentioned by \citep{kimconstrained}, this method of clipping could produce a new hyperparameter (the relative weight for this term relative to the Bellman residual). We keep this coefficient fixed (to unity) for simplicity, and we intend on exploring the possibility of a variable weight in future work.

\begin{equation}
    \textrm{BV} \doteq || \widetilde{Q}(s,a) - f(\{Q^{(k)}(s,a)\}_{k=1}^M) ||
    \label{eq:bound_violations}
\end{equation}

Similar to Eq. (21) of \citep{kimconstrained} we also considered a clipping at test-time only, with some differences in how the bounds are applied. This discrepancy is due to the difference in frameworks: \citep{kimconstrained} leverages the GPI framework, and in our setting we are learning a new policy from scratch while imposing said bounds. Our method is as follows:
Whenever the agent acts greedily and samples from the policy network, it first applies (hard) clipping to the network's value; then the agent extracts an action via greedy argmax. We term this method of clipping as ``test clipping''. The results for each method (as well as a combination of both hard and soft clipping) are shown in Fig.~\ref{fig:tasks} and \ref{fig:bound_viol}.

Interestingly, we find that by directly incorporating the bound violations into the loss function (via the ``soft'' clipping mechanism); the bound violations most quickly become (and remain) zero (Fig.~\ref{fig:bound_viol}) as opposed to the other methods considered. We find that reduction in bound violation also generally correlates with a high evaluation reward during training. One exception to this observation (for the particular environment shown) is the case of ``test'' clipping.

For this particular composition, either primitive task will solve the composite task, thus yielding high evaluation rewards (Fig.~\ref{fig:tasks}). However, the $Q$-values are not accurate, which leads to a high frequency of clipping, comparable to the baseline without clipping (Fig.~\ref{fig:bound_viol}).
In order to ensure the agent has learned accurate $Q$-values, it is therefore important to monitor the bound violations rather than only the evaluation performance which may not be representative of convergence of $Q$-values.

\section{Discussion}
In summary, we have established a general theoretical treatment for functional transformation and composition of primitive tasks. This extends the scope of previous work, which has primarily focused on isolated instances of reward transformations and compositions without general structure. Additionally, we have theoretically addressed the broader setting of stochastic dynamics, with rewards varying on both terminal and non-terminal (i.e. boundary and interior) states. In this work, we have shown that it is possible to derive a general class of functions which obey transfer bounds in standard and entropy-regularized RL beyond those cases discussed in previous work. In particular, we show that by using the same functional form on the optimal $Q$ functions as used on the reward, we can bound the transformed optimal $Q$ function. The derived bound can then be used to calculate a zero-shot solution. We have used these functions to define a Transfer Library: a set of tasks which can immediately be addressed by our bounds. Since our approach via the optimal backup equation is general, we apply it to both standard RL and entropy-regularized RL.

The newly-defined fixed point $C$ ($\hat{C}$) has an interesting interpretation. Rather than simply being an arbitrary function, for both the standard RL and entropy-regularized RL bounds, $C$ represents an optimal value function for a standard RL task with reward function given by $r_C$ ($\hat{r}_C$ for $\hat{C}$). 

The function $C$ bounds the total gap between $f(Q(s,a))$ and $\widetilde{Q}(s,a)$ at the level of state-actions. We also note the simple relationship between reward functions $r_C = -\hat{r}_C$.

The fixed point $D$ ($\hat{D}$) is not an optimal value function, but the value of the zero-shot policy $\pi_f$ in some other auxiliary task. The auxiliary task takes various ``rewards'', e.g. the function $\gamma \hat{C}$ in Lemma \ref{lem:concave_regret_maxent}. Although for general functions $f$, the rewards do not have a simple interpretation (i.e. R\'enyi divergence between two policies as in \citep{Haarnoja2018}), we see that $r_C$ essentially measures the non-linearities of the composition function $f$ with respect to the given dynamics, and hence accounts for the errors made in using the bounding conditions of $f$.
Furthermore, we can bound $C$ (and thus the difference between the optimal value and the suggested zero-shot approximation $f(Q)$) in a simple way: by bounding the rewards corresponding to $C$. By simply calculating the maximum of $r_C$ for example, one easily finds $C(s,a) \le \frac{1}{1-\gamma} \max_{s,a} r_C(s,a)$ (and similarly for $\hat{C}$).

Interestingly, \cite{TL_bound} have shown the provable usefulness of using an upper bound when used for ``warmstarting'' the training in new domains. In particular, it appears that $f(Q)$ (for the concave conditions) is related to their proposed ``$\alpha$-weak admissible heuristic'' for $\widetilde{\T}$. In future work, we hope to precisely connect to such theoretical results in order to obtain provable benefits to our derived bounds.
Experimentally we have observed that this warmstarting procedure does indeed improve convergence times, however a detailed study of this effect is beyond the scope of the present work and will be explored in future work. The derived results have also been used to devise protocols for clipping which improve performance and reduce variance during training based on the experiments presented.

In the future, we hope that the class of functions discussed in this work will be broadened further, allowing for a larger class of non-trivial zero-shot bounds for the Transfer Library. By adding these known transformations and compositions to the Transfer Library, the RL agent will be able to approach significantly more novel tasks without the need for further training. 

The current research has also emphasized questions for transfer learning in this context, such as:
\textit{Which primitives should be prioritized for learning?} (Discussed in \cite{boolean_stoch, nemecek, alver}.) \textit{What other functions can be used for transfer?} \textit{How tight are these bounds?} \textit{How does the Transfer Library depend on the parameters $\gamma$ and $\beta$?}

In this work, we provide general bounds for the discrete MDP setting and an extension of the theory to continuous state-action spaces is deferred to future work. It would be of interest to explore if it is possible to prove general bounds for this extension, given sufficient smoothness conditions on the dynamics and the function of transformation. Other extensions can be considered as well, for instance: the applicability to other value-based or actor-critic methods, the warmstarting of function approximators, learning the $C$ and $D$ functions, and adjusting the ``soft'' clipping weight parameter.

In future work, we also aim to discover other functions satisfying the derived conditions; and will attempt to find necessary (rather than sufficient) conditions that classify the functions $f \in \TL$.
It would be of interest to explore if extensions of the current approach can further enable agents to expand and generalize their knowledge base to solve complex dynamic tasks in Deep RL.

\begin{acknowledgements} 
The authors would like to thank the anonymous reviewers
for their helpful comments and suggestions. JA, AA, and
RVK acknowledge funding support from the NSF through
Award No. DMS-1854350. VM and ST acknowledge funding support from the NSF through Award No. 2246221.
JA would like to acknowledge the use of the supercomputing facilities managed by the Research Computing Department at the University of Massachusetts Boston. The
work of JA and AA was supported in part by the College
of Science and Mathematics Dean’s Doctoral Research Fellowship through fellowship support from Oracle, project ID
R20000000025727. JA and RVK would like to acknowledge
support from the Proposal Development Grant provided by
the University of Massachusetts Boston. ST and VM acknowledge support from the Alliance Innovation Lab in
Silicon Valley.
\end{acknowledgements}

\bibliography{uai2023-template}
\nocite{haarnoja2019_learn}
\nocite{hardy1952inequalities}
\nocite{lee2021sunrise}
\nocite{openAI}

\title{Bounding the Optimal Value Function \\in Compositional Reinforcement Learning\\(Supplementary Material)}

\onecolumn
\maketitle

\section{Introduction}
In the following, we discuss the results of additional experiments in the four room domain. In these experiments, we want to answer the following questions:
\begin{itemize}
    \item How do the optimal policies and value functions compare to those calculated from the zero-shot approximations using the derived bounds?
    \item What are other examples of compositions and functional transformations that can be analyzed using our approach?
    \item Does warmstarting (using the derived bounds for initialization) in the tabular case improve the convergence?

\end{itemize}

To address these issues, we modify OpenAI's frozen lake environment \cite{openAI} to allow for stochastic dynamics. 

In the tabular experiments, numerical solutions for the optimal $Q$ functions were obtained by solving the Bellman backup equations iteratively. Iterations are considered converged once the maximum difference between successive iterates is less than $10^{-10}$.


Beyond the motivating example shown in the main text, we have included video files demonstrating a full range of zero-shot compositions with convex weights between the Bottom Left (BL) room and Bottom Right (BR) room subtasks, in both entropy-regularized ($\beta=5$) and standard RL with deterministic dynamics. These videos, along with all code for the above experiments are made publicly available at a repository on \url{https://github.com/JacobHA/Q-Bounding-in-Compositional-RL}.

\section{Experiments}
\subsection{Function Approximators}

For function approximator experiments (as shown in the main text), we use the DQN implementation from Stable-Baselines3 \cite{stable-baselines3}. We first fully train the subtasks (seen in Fig. 1 of the main text). Then, we perform hyperparameter sweeps for each possible clipping option. Several hyperparameters are kept fixed (Table~\ref{tab:shared}), and we sweep with the range and distribution shown below in Table~\ref{tab:sweep}. Finally, we use the optimal hyperparameters (as measured by those which maximize the accumulated reward throughout training). These values are shown in Table~\ref{tab:optimal}. 
\clearpage
\begin{center}
\captionof{table}{Hyperparameters shared by all Deep Q Networks}
\begin{tabular}{||c c||} 
 \hline
 Hyperparameter & Value \\ [0.5ex] 
 \hline\hline
 Buffer Size & 1,000,000 \\ 
 \hline
 Discount factor, $\gamma$ & 0.99 \\
  \hline
 $\epsilon_{\text{initial}}$ & 1.0 \\
 \hline
 $\epsilon_{\text{final}}$ & 0.05 \\
  \hline
 ``learning starts'' & 5,000 \\
 \hline
 Target Update Interval & 10000 \\ [1ex] 
 \hline
\end{tabular}
\label{tab:shared}
\end{center}

\begin{center}
\captionof{table}{Hyperparameter Ranges Used for Finetuning}
\begin{tabular}{||c c c c||} 
 \hline
 Hyperparameter & Sampling Distribution & Min. Value & Max. Value\\ [0.5ex] 
  \hline
 Learning Rate & Log Uniform & $10^{-4}$ & $10^{-1}$ \\ 
 \hline
 Batch Size & Uniform & $32$ & $256$ \\ 
 \hline
 Exploration Fraction & Uniform & $0.1$ & $0.3$ \\
 \hline
 Polyak Update, $\tau$ & Uniform & $0.5$ & $1.0$ \\ [1ex] 
 \hline
\end{tabular}
\label{tab:sweep}
\end{center}

\begin{center}
\captionof{table}{Hyperparameters used for different clipping methods}
\begin{tabular}{||c c c c c||} 
 \hline
 Hyperparameter & None & Soft & Hard & Soft-Hard \\ [0.5ex] 
  \hline
 Learning Rate & $7.825\times10^{-4}$ & $3.732\times10^{-3}$ & $1.457\times10^{-3}$  & $3.184\times10^{-3}$ \\ 
 \hline
 Batch Size & 245 & 247 & 146 & 138\\ 
 \hline
 Exploration Fraction & 0.137 & 0.1075 & 0.1243 & 0.1207\\
 \hline
 Polyak Update, $\tau$ & 0.9107 & 0.9898 & 0.5545 & 0.7682 \\ [1ex] 
 \hline
\end{tabular}
\label{tab:optimal}
\end{center}

\subsection{Tabular experiments}
In these experiments we will demonstrate on simple discrete environments the effect of increasingly stochastic dynamics and increasingly dense rewards. As a proxy for measuring the usefulness or accuracy of the bound $f(Q)$, we calculate the mean difference between $f\left(Q(s,a)\right)-\widetilde{Q}(s,a)$, as well as the mean Kullback-Liebler (KL) divergence between $\pi$ (the true optimal policy) and $\pi_f$, the policy derived from the bound $f(Q)$. The proceeding experiments are situated in the entropy-regularized formalism (unless $\beta=\inf$ as shown in Fig.~\ref{fig:beta_inf_stoch_expt}) with the uniform prior policy $\pi_0(a|s) = 1/ |\mathcal{A}|$.

\subsubsection{Stochasticity of Dynamics}
In this experiment, we investigate the effect of stochastic dynamics on the bounds. Specifically, we vary the probability that taking an action will result in the intended action. This is equivalent to a slip probability.

\begin{figure}[ht]
    \centering
    \subfloat[Task 1]{{\includegraphics[width=5cm]{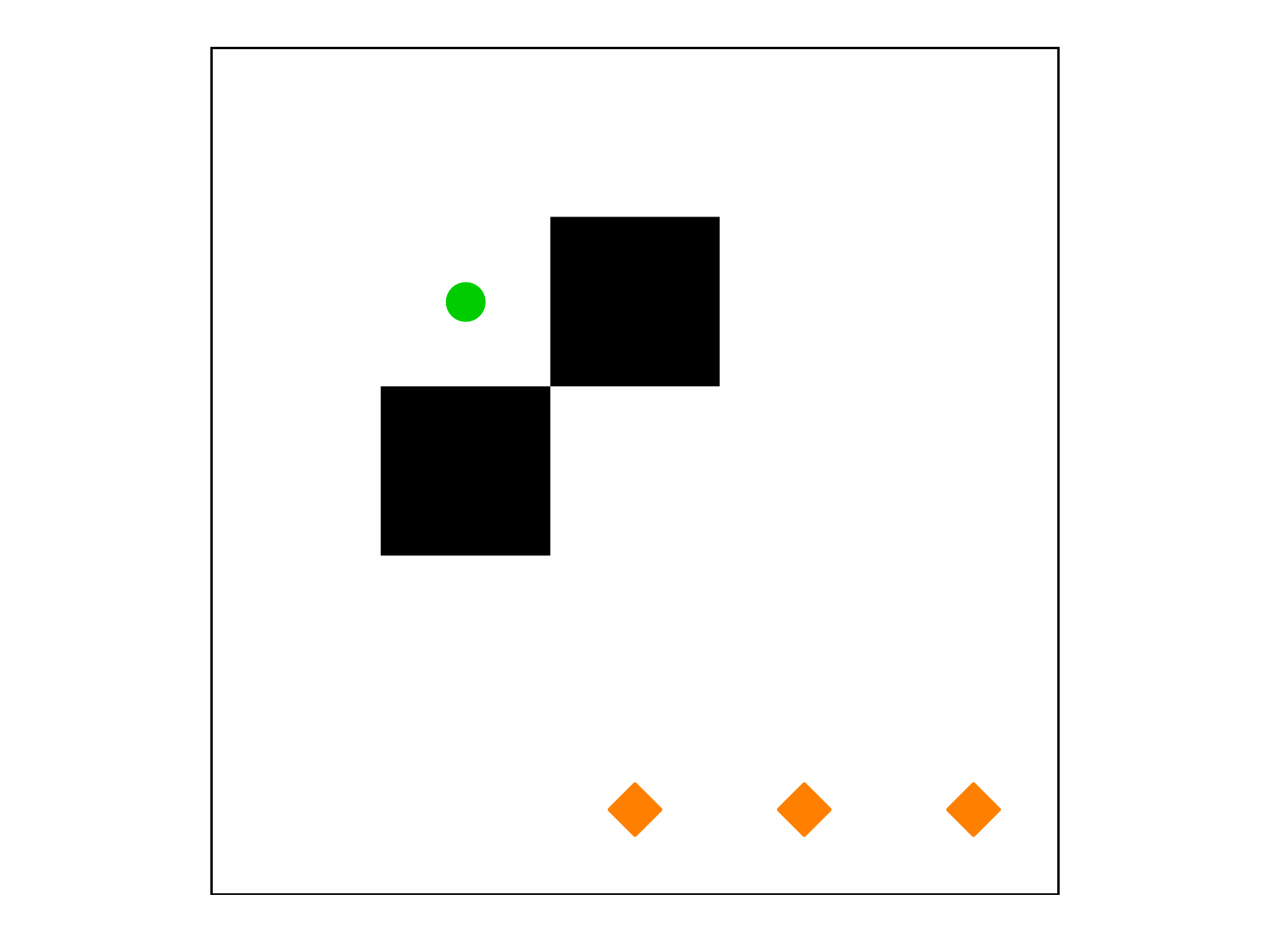} }}
    \subfloat[Task 2]{{\includegraphics[width=5cm]{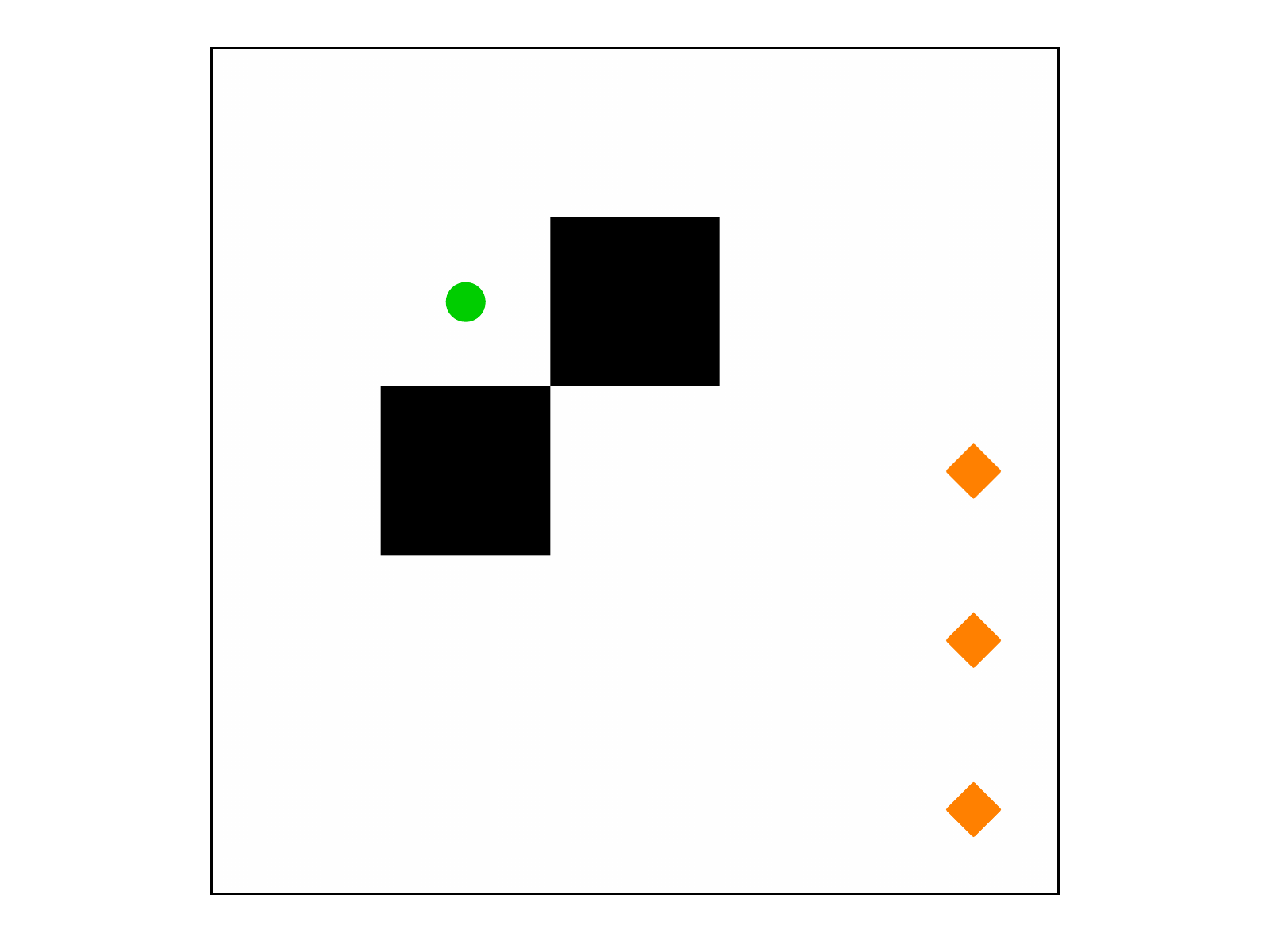} }}
    \subfloat[Task 3: AND composition]{{\includegraphics[width=5cm]{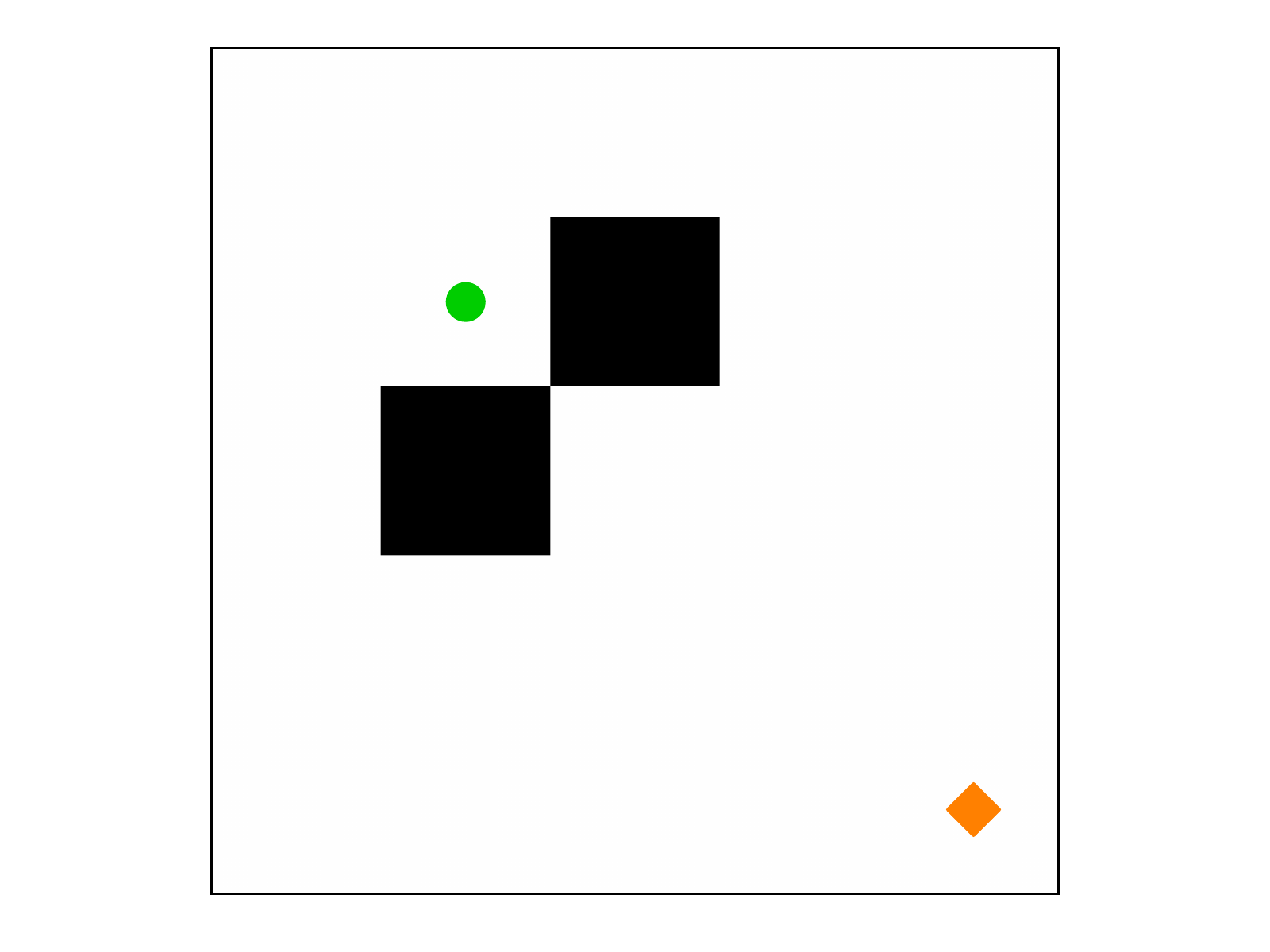} }}
    \caption{Reward functions for a simple maze domain; used for stochasticity experiments. We place reward (whose cost is half the default step penalty of $-1$) at the edges of the room, denoted by an orange diamond. }
    \label{fig:stochastic_desc}
\end{figure}
We notice in the following plots that at near-deterministic dynamics the bound becomes tighter. We also remark that the Kullback-Liebler divergence is lowest in very highly-stochastic environments. This is because for any $\beta>0$, the cost of changing the policy $\pi$ away from the prior policy is not worth it: the dynamics are so stochastic that there will be no considerable difference in trajectories even if significant controls (nearly deterministic choices) are applied via $\pi$.

\begin{figure}[ht]
    \centering
    \subfloat[]{{\includegraphics[width=7cm]{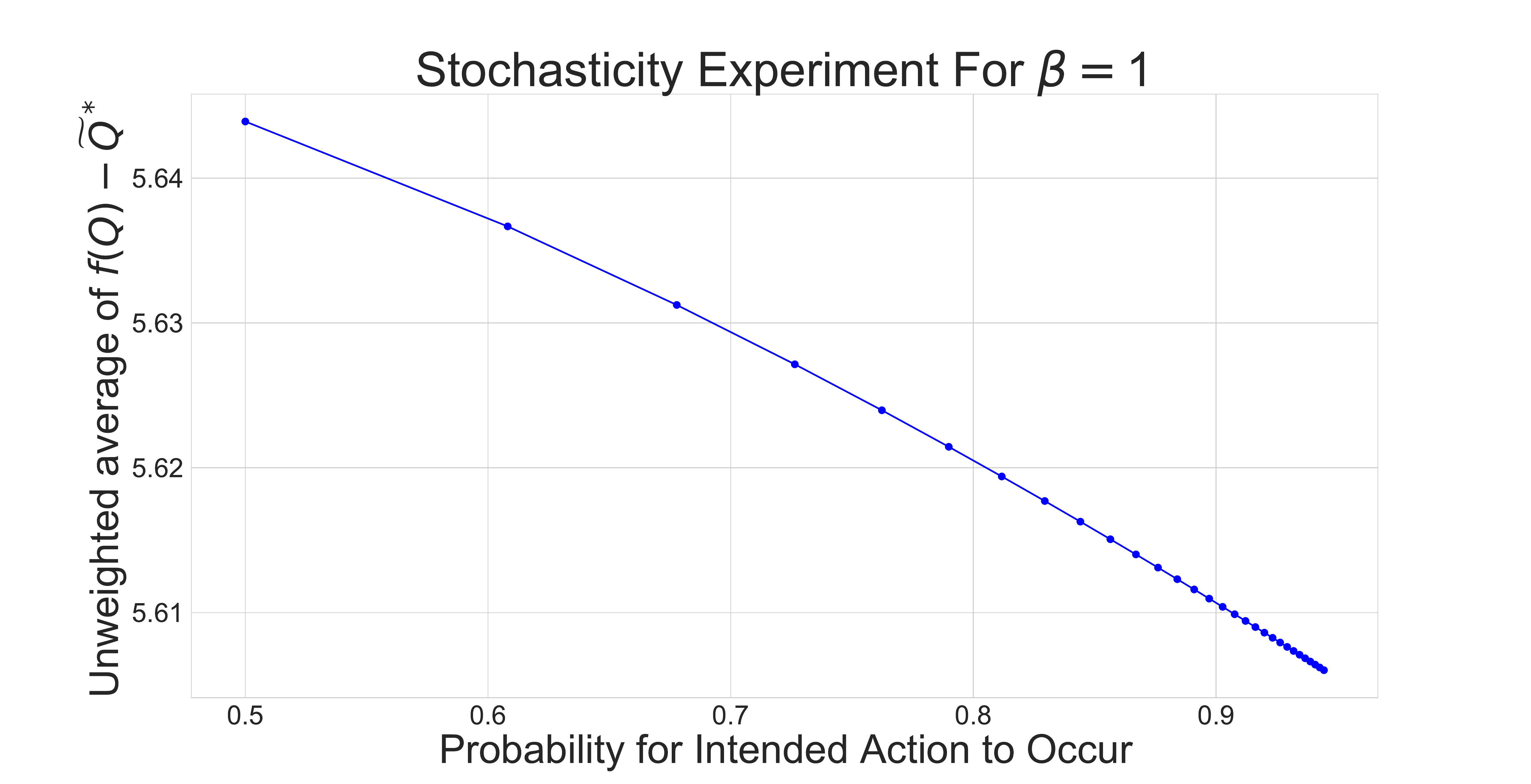} }}
    \subfloat[]{{\includegraphics[width=7cm]{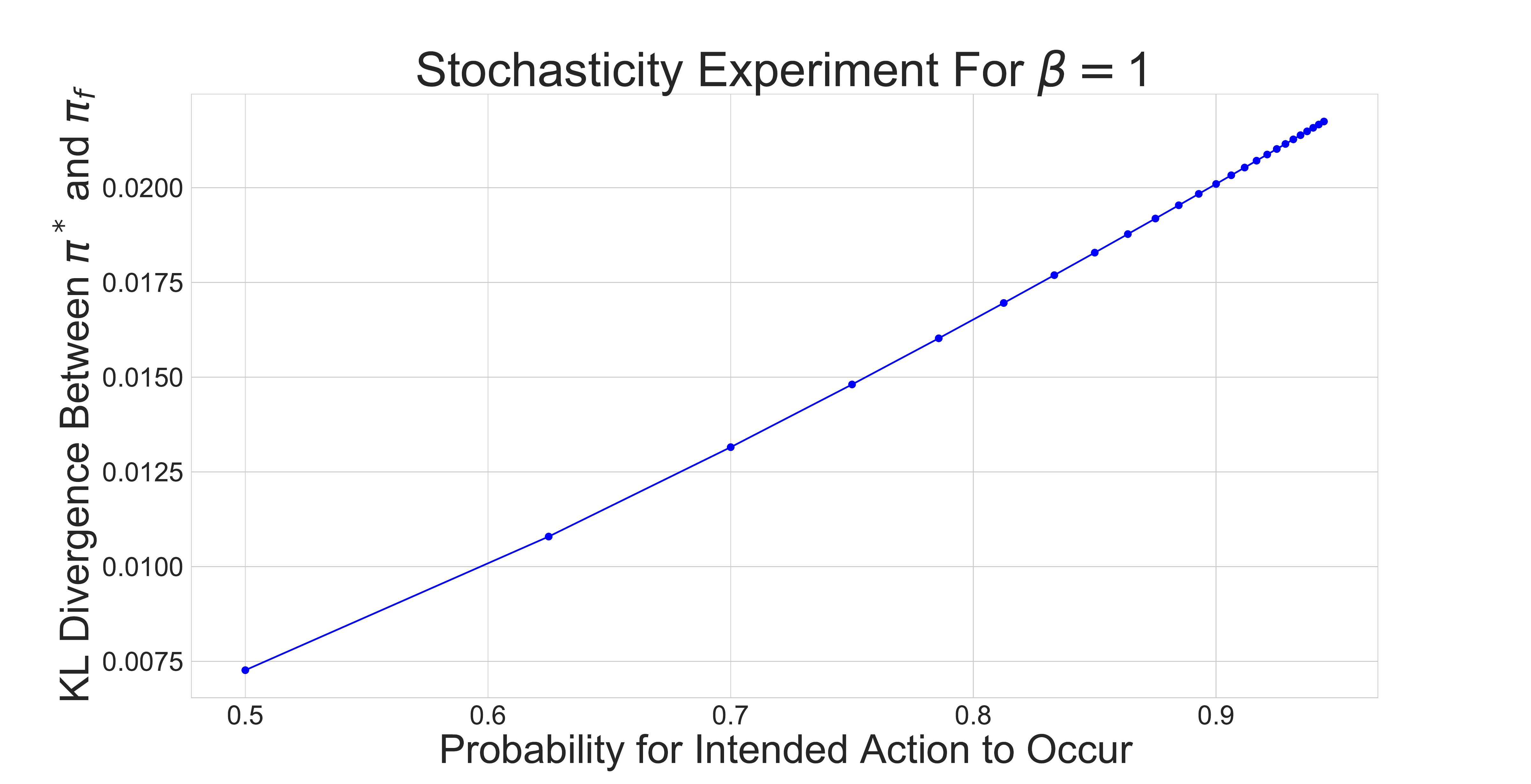} }}

    \caption{$\beta=1$ KL divergence between $\pi$ and $\pi_f$ and average difference between optimal $Q$ function and presented bound.}
    \label{fig:b1}
\end{figure}

\begin{figure}[ht]
    \centering
    \subfloat[]{{\includegraphics[width=8cm]{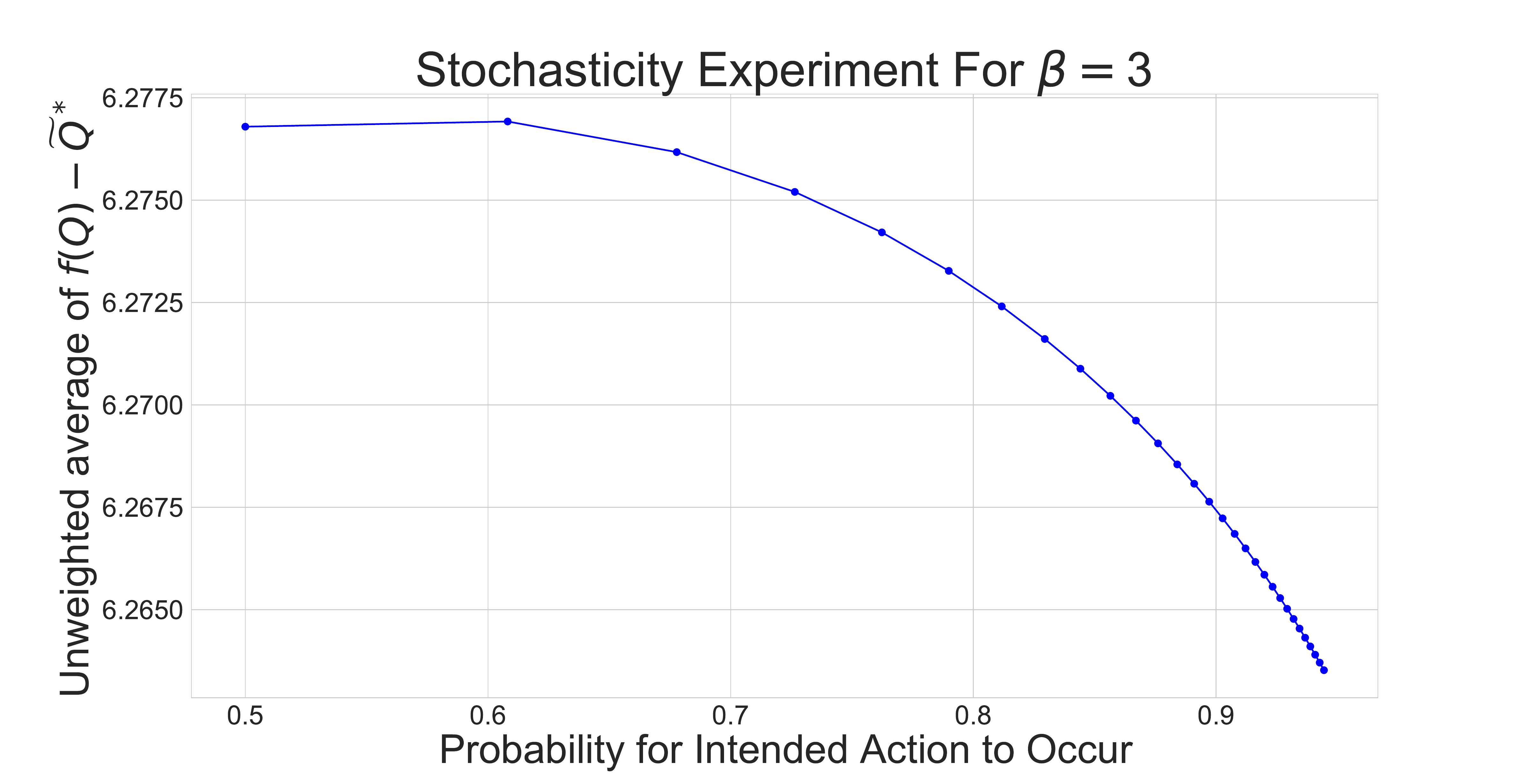} }}
    \subfloat[]{{\includegraphics[width=7cm]{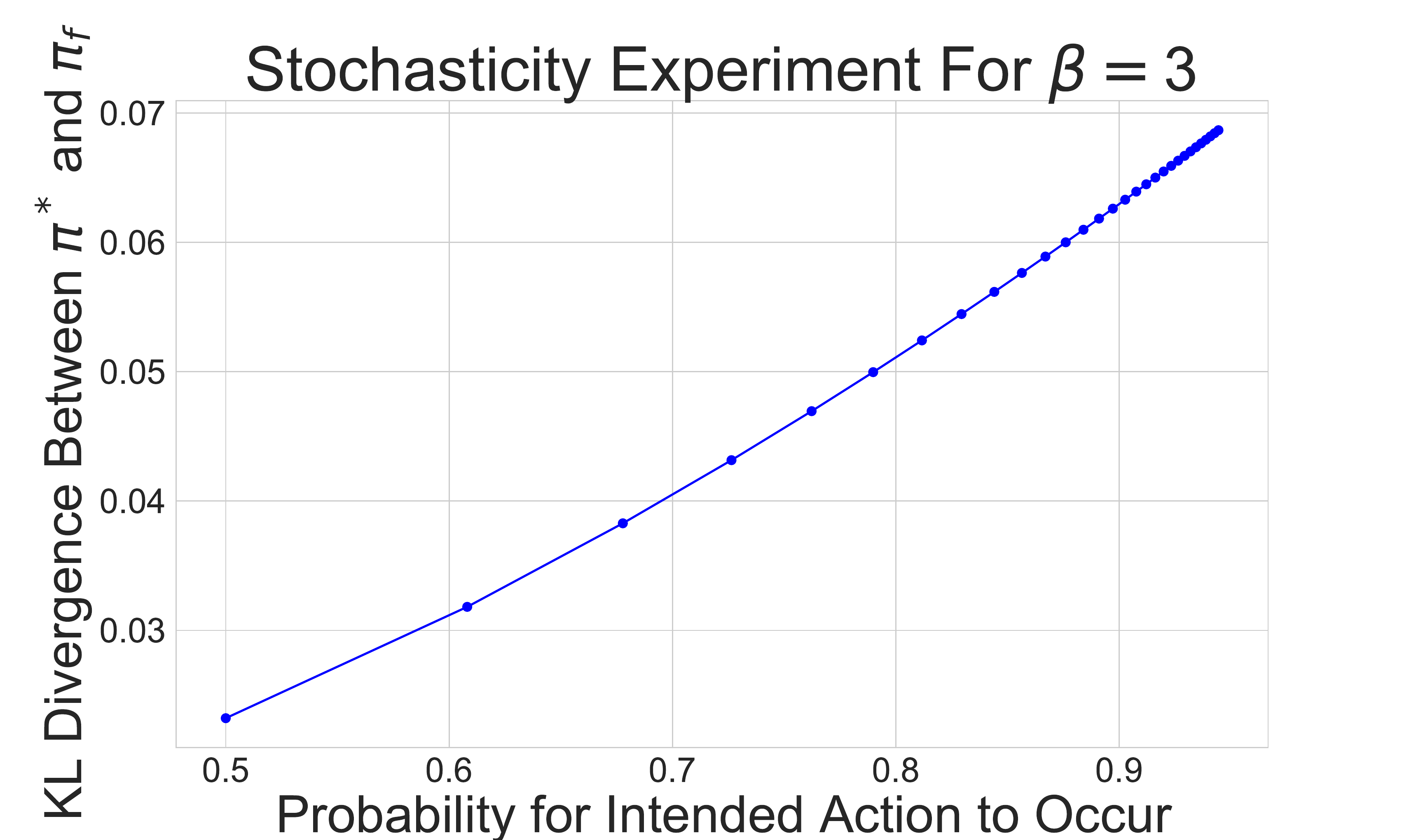} }}

    \caption{$\beta=3$ KL divergence between $\pi$ and $\pi_f$ and average difference between optimal $Q$ function and presented bound.}
    \label{fig:b3}
\end{figure}

\begin{figure}[ht]
    \centering
    \subfloat[]{{\includegraphics[width=7cm]{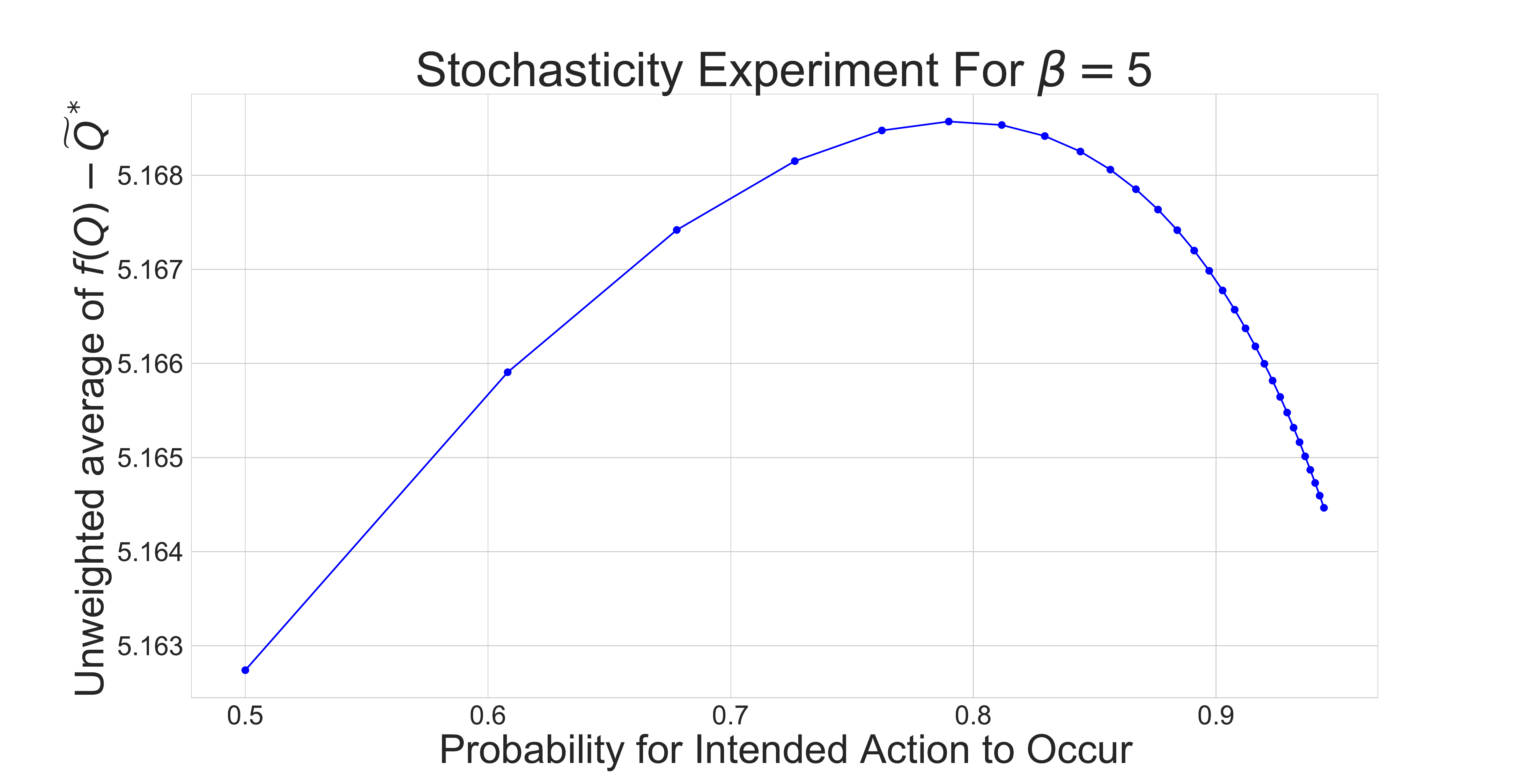} }}
    \subfloat[]{{\includegraphics[width=7cm]{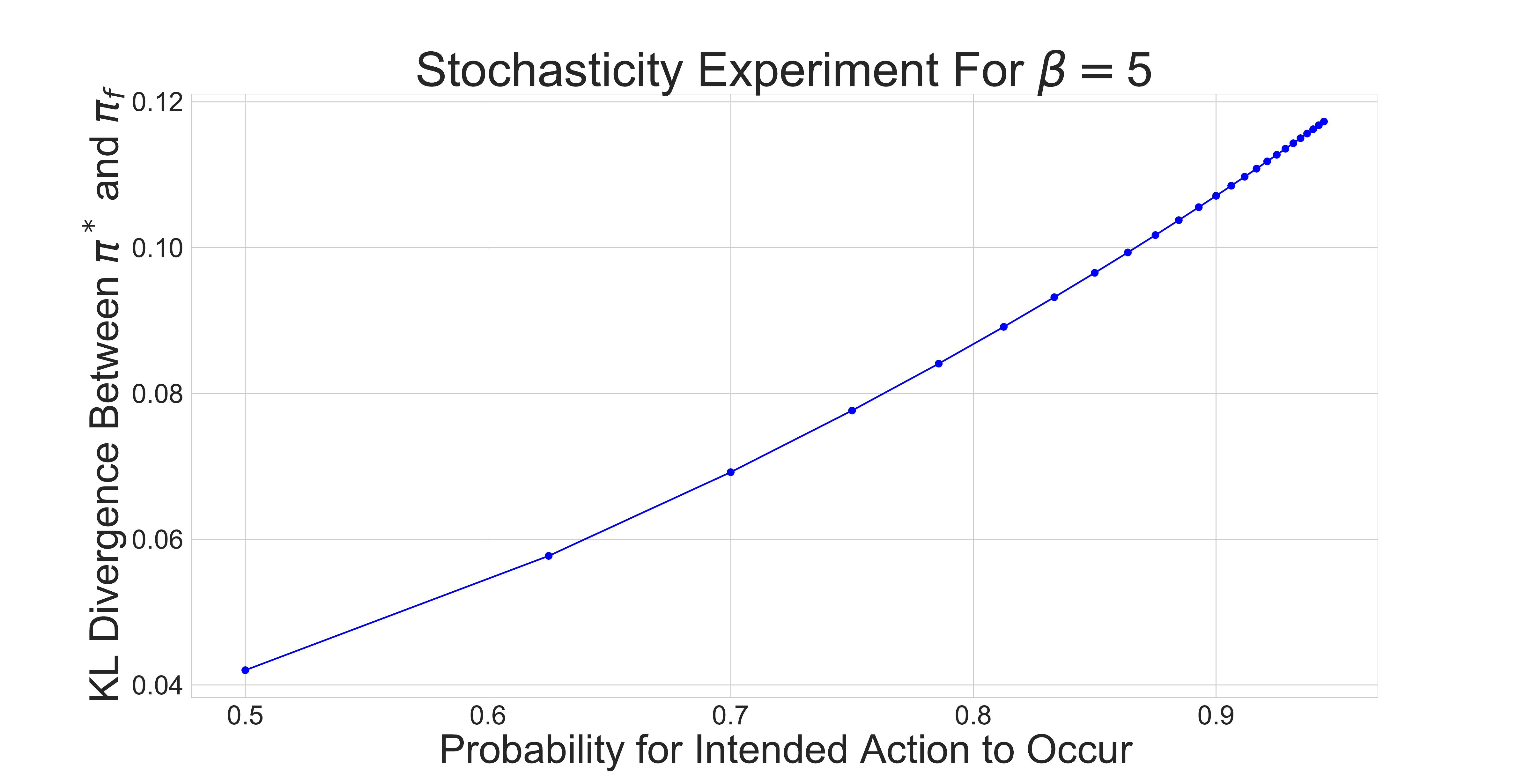} }}

    \caption{$\beta=5$ KL divergence between $\pi$ and $\pi_f$ and average difference between optimal $Q$ function and presented bound.}
    \label{fig:b5}
\end{figure}

\begin{figure}[ht]
    \centering
    \includegraphics[width=12cm]{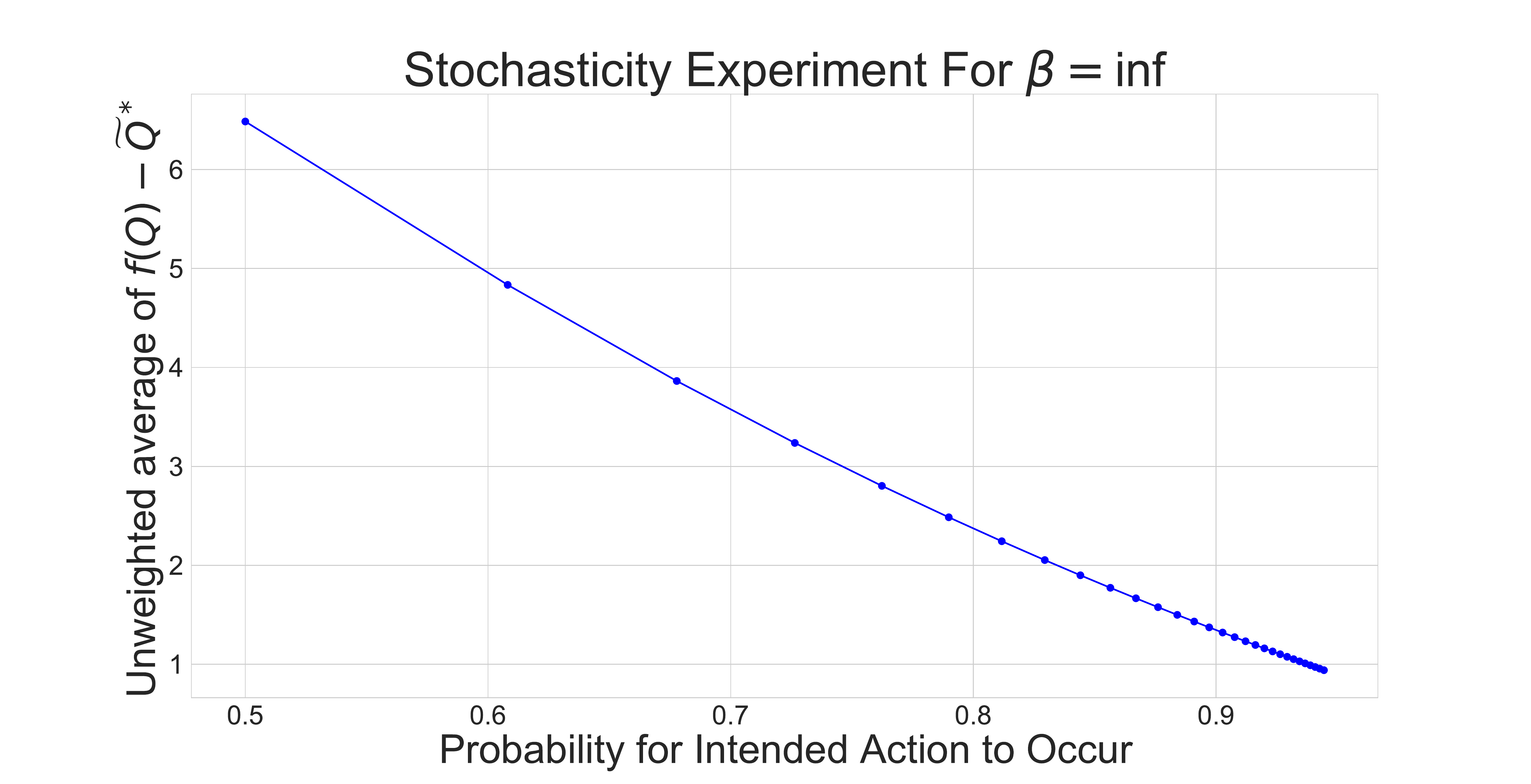} 

    \caption{$\beta=\inf$, standard RL. 
    Average difference between optimal $Q$ function and presented bound. Note that we do not plot a KL divergence for this case as $\pi$ is greedy and hence the divergence is always infinite.}
    \label{fig:beta_inf_stoch_expt}
\end{figure}

\clearpage

\subsection{Sparsity of Rewards}
In this experiment, we consider an empty environment ($|S|\times |S|$ empty square) with reward $r=0$ everywhere and deterministic dynamics. No other rewards or obstacles are present. Then fix an integer $0<n<|S|$. Drawing randomly (without repetition), we choose one of the states of the environment to grant a reward, drawn uniformly between $(0, 1)$. We do this again for another copy of the empty environment.

We then compose these two (randomly generated as described) subtasks by using a simple average $F(r^{(1)}, r^{(2)}) = 0.5r^{(1)} + 0.5 r^{(2)}$. We have used $\beta=5$ for all experiments in this subsection.
\begin{figure}[ht]
    \centering
    \subfloat[]{{\includegraphics[width=7cm]{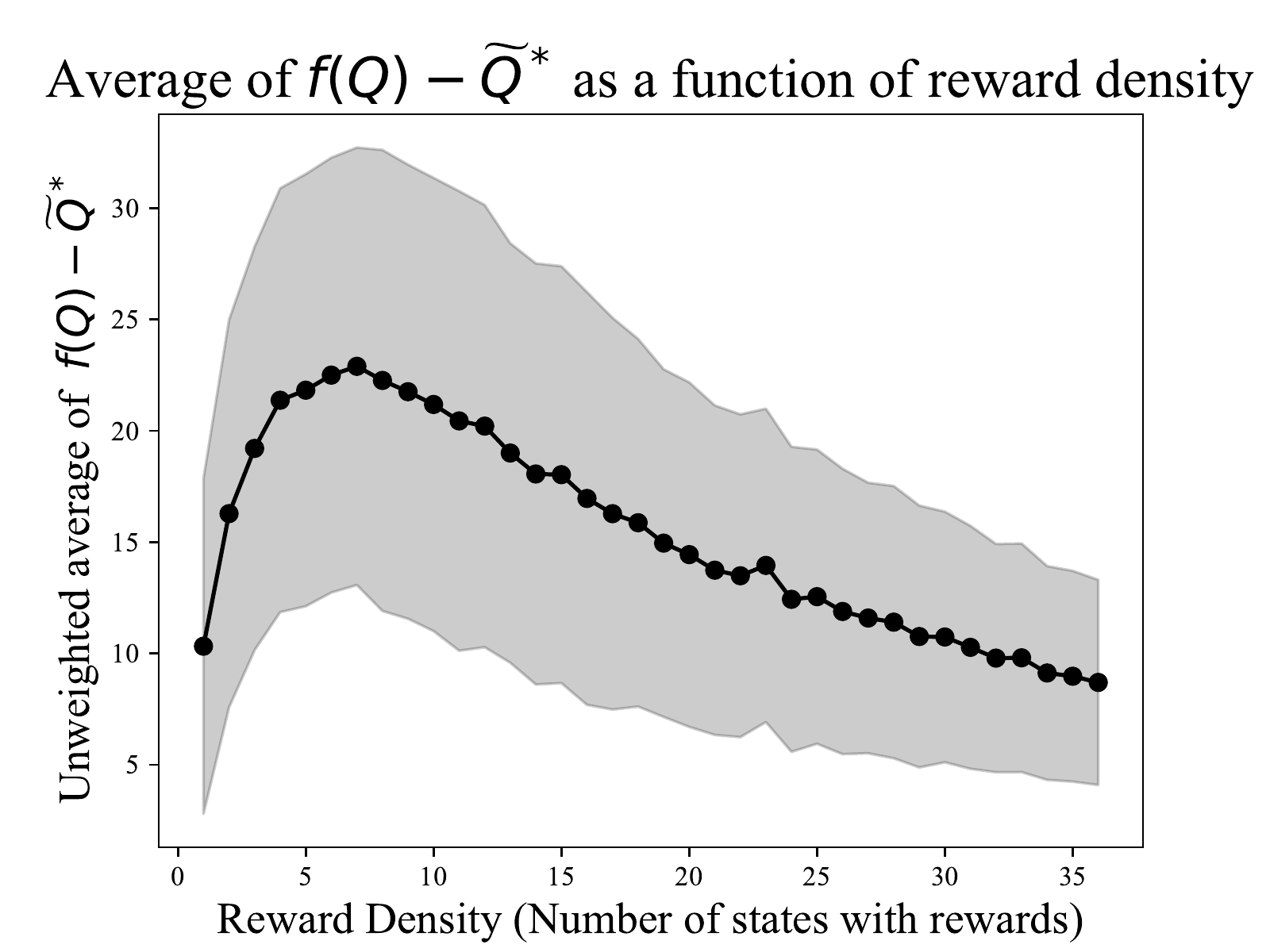} }}
    \subfloat[]{{\includegraphics[width=7cm]{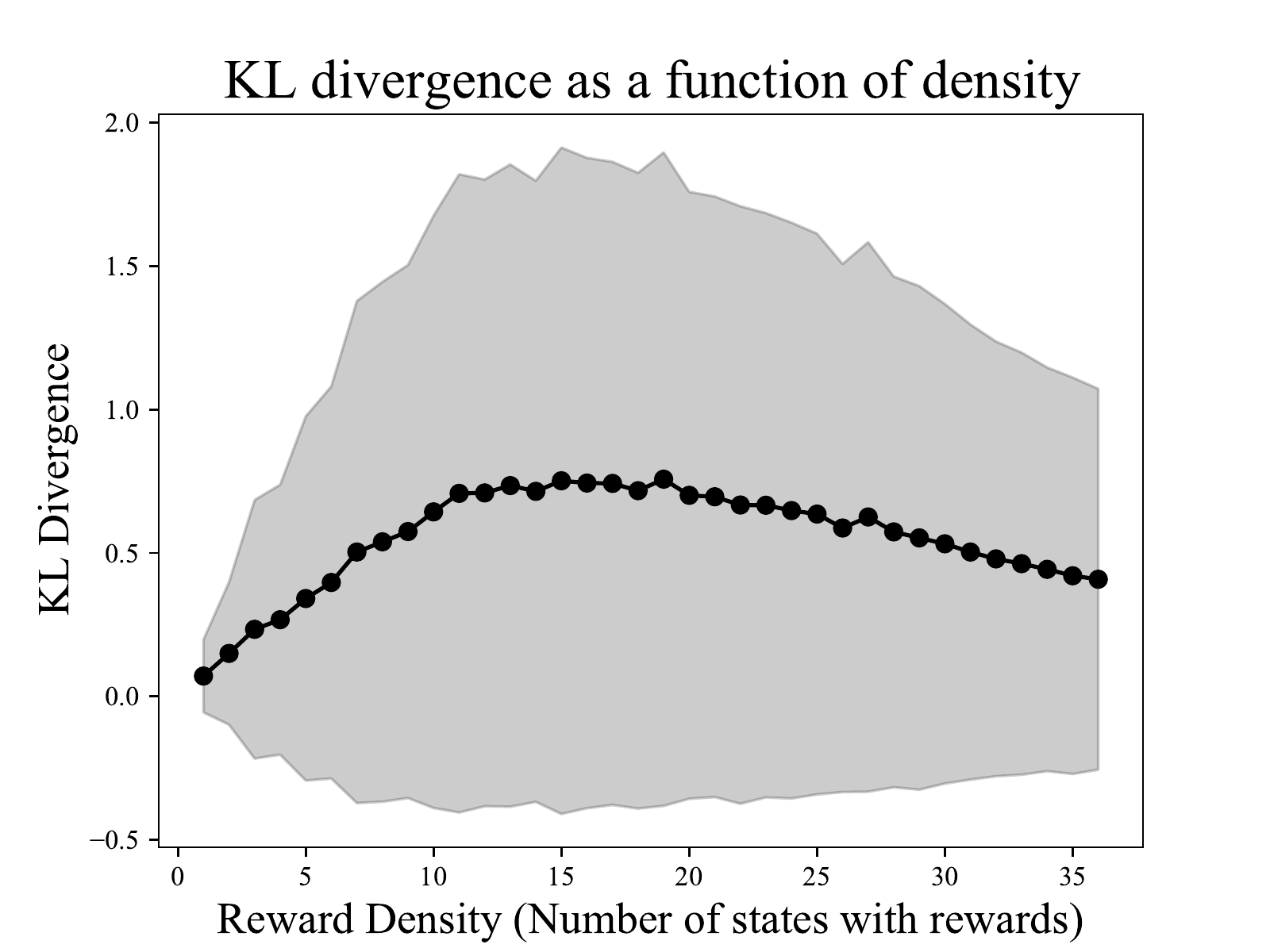} }}

    \caption{$6\times 6$ environment. KL divergence between $\pi$ and $\pi_f$ and average difference between optimal $Q$ function and presented bound, with the shaded region representing one standard deviation over 1000 runs.}
    \label{fig:6x6}
\end{figure}

\begin{figure}[ht]
    \centering
    \subfloat[]{{\includegraphics[width=7cm]{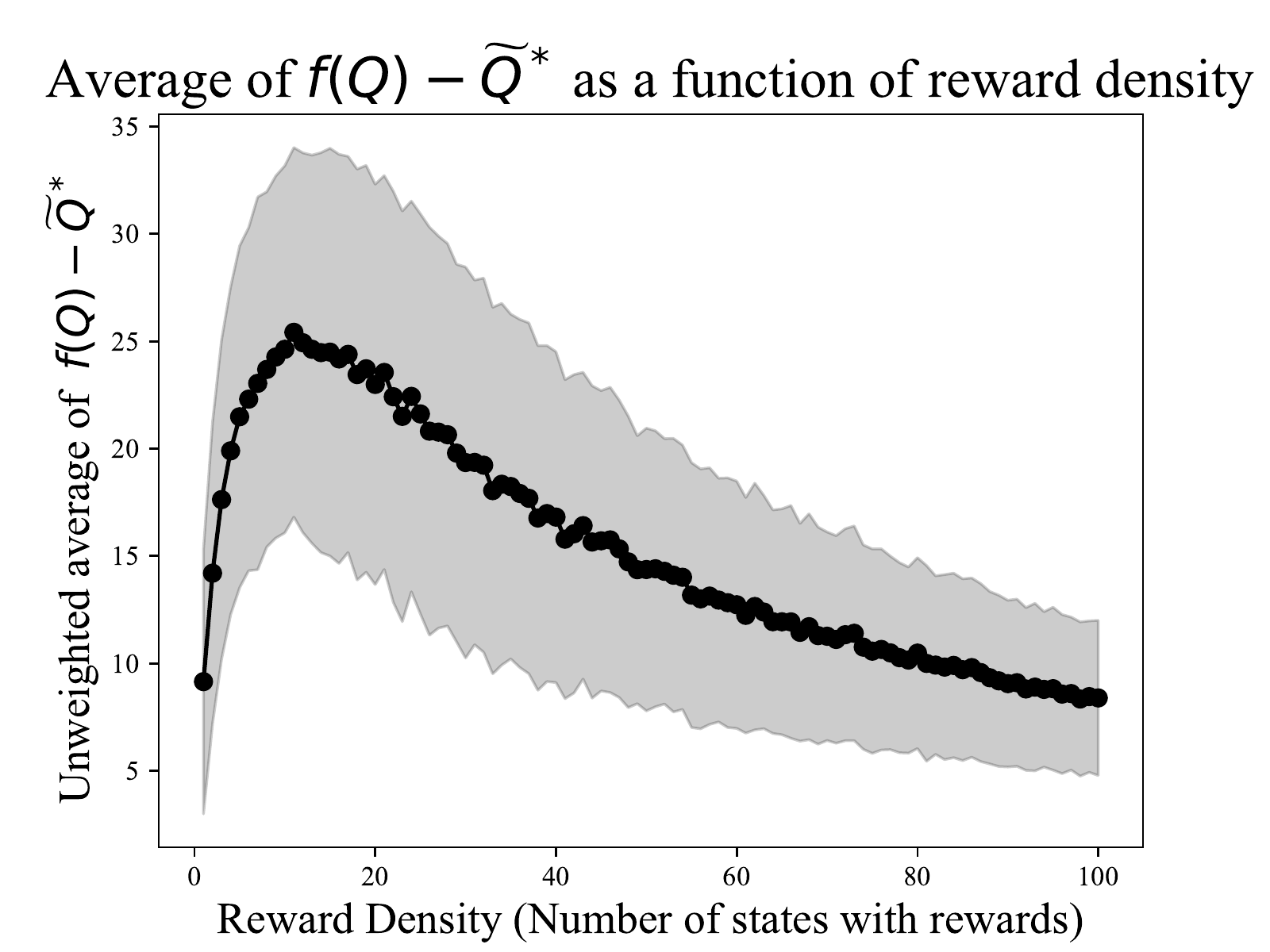} }}
    \subfloat[]{{\includegraphics[width=7cm]{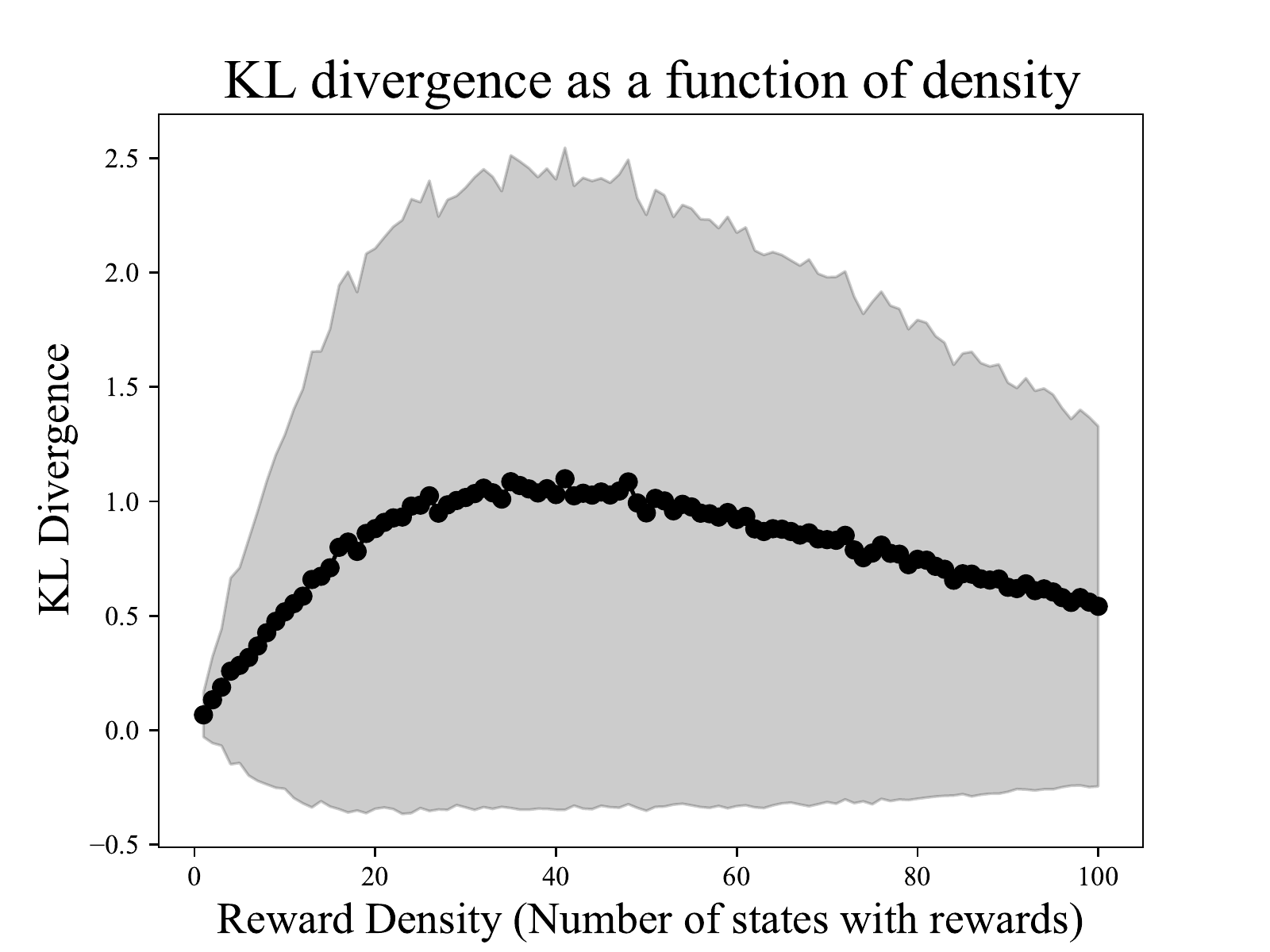} }}

    \caption{$10\times 10$ environment. KL divergence between $\pi$ and $\pi_f$ and average difference between optimal $Q$ function and presented bound, with the shaded region representing one standard deviation over 1000 runs.}
    \label{fig:10x10}
\end{figure}
Interestingly, we find a somewhat universal behavior, in that there is a certain level of density which makes the bound a poor approximation to the true $Q$ function. We also note that the bound is a better approximation at low densities.

\clearpage
\newpage
\section{Boolean Composition Definitions}
In this section, we explicitly define the action of Boolean operators on subtask reward functions. These definitions are similar to those used by \cite{boolean}.

\begin{definition}[OR Composition]
Given subtask rewards $\{r^{(1)}, r^{(2)}, \dotsc , r^{(M)} \}$, the OR composition among them is given by the \textit{maximum} over all subtasks, at each state-action pair:
\begin{equation}
    r^{(\text{OR})}(s,a) = \max_k r^{(k)}(s,a).
\end{equation}
\end{definition}

\begin{definition}[AND Composition]
Given subtask rewards $\{r^{(1)}, r^{(2)}, \dotsc , r^{(M)} \}$, the AND composition among them is given by the \textit{minimum} over all subtasks, at each state-action:
\begin{equation}
    r^{(\text{AND})}(s,a) = \min_k r^{(k)}(s,a).
\end{equation}
\end{definition}

\begin{definition}[NOT Gate]
Given a subtask reward function $r$, applying the NOT gate transforms the reward function by negating all rewards (i.e. rewards $\to$ costs):
\begin{equation}
    r^{(\text{NOT})}(s,a) = - r(s,a),
\end{equation}
\end{definition}

The proofs in all subsequent sections follow an inductive form based on the Bellman backup equation, whose solution converges to the optimal $Q$ function. This is a similar approach as employed by \cite{Haarnoja2018} and \cite{hunt_diverg}, but with the extension to all applicable functions; rather than (linear) convex combinations.

\section{Proofs for Standard RL}

Let $X$ be the codomain for the $Q$ function of the primitive task ($Q: \s \times \A \to X \subseteq \mathbb{R}$). 

\begin{lemma}[Convex Conditions]\label{thm:convex_cond_std}
Given a primitive task with discount factor $\gamma$ and a bounded, continuous transformation function $f~:~X~\to~\mathbb{R}$ which satisfies:  
\begin{enumerate}
    \item $f$ is convex on its domain $X$ (for stochastic dynamics);
    \item $f$ is sublinear:
    \begin{enumerate}[label=(\roman*)]
        \item $f(x+y) \leq f(x) + f(y)$ for all $x,y \in X$
        \item $f(\gamma x) \leq \gamma f(x)$ for all $x \in X$ 
    \end{enumerate}
    \item $f\left( \max_{a} \mathcal{Q}(s,a) \right) \leq \max_{a}~f\left( \mathcal{Q}(s,a) \right)$ for all $\mathcal{Q}: \s \times \A \to \mathbb{R}.$
\end{enumerate}

then the optimal action-value function for the transformed rewards, $\widetilde{Q}$, is now related to the optimal action-value function with respect to the original rewards  by:

\begin{equation}\label{eqn:convex_std}
    f(Q(s,a)) \leq \widetilde{Q}(s,a) \leq f(Q(s,a)) + C(s,a)
\end{equation}

where $C$ is the optimal value function for a task with reward
\begin{equation}\label{eq:std_convex_C_def}
    r_C(s,a) = f(r(s,a)) + \gamma \mathbb{E}_{s'} V_f(s') - f(Q(s,a)).
\end{equation}

\end{lemma}

\begin{proof}
We will prove all inequalities by induction on the number of backup steps, $N$. We start with the lower bound $\widetilde{Q} \ge f(Q)$. The base case, $N=1$ is trivial since $f(r(s,a))=f(r(s,a))$. The inductive step is the assumption $\widetilde{Q}^{(N)}(s,a) \geq f(Q^{(N)}(s,a))$ for some $N>1$.
In the case of standard RL, the Bellman backup equation for transformed rewards is given by: 

\begin{equation}
    \widetilde{Q}^{(N+1)}(s,a) = f\left(r(s,a)\right) + \gamma \mathbb{E}_{s' \sim{} p(s'|s,a)} \max_{a'} \widetilde{Q}^{(N)}(s',a')
\end{equation}
Using the inductive assumption,
\begin{equation}
    \widetilde{Q}^{(N+1)}(s,a) \geq f\left(r(s,a)\right) + 
    \gamma \mathbb{E}_{s' \sim{} p(s'|s,a)} \max_{a'} f \left({Q}^{(N)}(s',a') \right)
\end{equation}

The condition $v_f(s) \ge f(v(s)) $ is used on the right hand side to give:
\begin{equation}
    \widetilde{Q}^{(N+1)}(s,a) \geq f\left(r(s,a)\right) + \gamma \mathbb{E}_{s' \sim{} p(s'|s,a)} f \left( \max_{a'} {Q}^{(N)}(s',a') \right)
\end{equation}
Since $f$ is convex, we use Jensen's inequality to factor it out of the expectation. Note that this condition on $f$ is only required for stochastic dynamics. The error introduced by swapping these operators is characterized by the ``Jensen's gap'' for the transformation function $f$.
\begin{equation}
    \widetilde{Q}^{(N+1)}(s,a) \geq f\left(r(s,a)\right) + \gamma f \left( \mathbb{E}_{s' \sim{} p(s'|s,a)}  \max_{a'} {Q}^{(N)}(s',a') \right)
\end{equation}
Finally, using both sublinearity conditions
\begin{equation}
    \widetilde{Q}^{(N+1)}(s,a) \geq f\left(r(s,a) + \gamma \mathbb{E}_{s' \sim{} p(s'|s,a)}  \max_{a'} {Q}^{(N)}(s',a') \right)
    \label{eq:pf:last_line_induction}
\end{equation}

where the right-hand side is simply $f(Q^{(N+1)}(s,a))$. Since this inequality holds for all $N$, we take the limit $N \to \infty$ wherein $Q^{(N)}$ converges to the optimal $Q$-function. For the right-hand side of Eq. \eqref{eq:pf:last_line_induction}, we thus have (by continuity of $f$):

\begin{equation}
    \lim_{N \to \infty} f\left(Q^{(N)}(s,a)\right) =  f\left(\lim_{N \to \infty}Q^{(N)}(s,a)\right) = f(Q(s,a))
\end{equation}
where $Q(s,a)$ is the optimal action value function for the primitive task. 
Combined with the limit of the left-hand side, we arrive at the desired inequality:
\begin{equation}
    \widetilde{Q}(s,a) \geq f\left(Q(s,a)\right).
\end{equation}

This completes the proof of the lower bound. To prove the upper bound we again use induction on the backup equation of $\widetilde{Q}^{(N)}$. We wish to show $\widetilde{Q}^{(N)} \le f\left(Q(s,a)\right) + C^{(N)}(s,a)$ holds for all $N$, with the definition of $C$ provided in Lemma~4.1.

    Let $f$ satisfy the convex conditions. 
    Consider the backup equation for $\widetilde{Q}$.
    Again, the base case ($N=1$) is trivially satisfied with equality. Using the inductive assumption, we find

\begin{align*}
    \widetilde{Q}^{(N+1)}(s,a) &= f(r(s,a)) + \gamma \E_{s'} \max_{a'} \widetilde{Q}^{(N)}(s',a')
\\
    &\le f(r(s,a)) + \gamma \E_{s'} \max_{a'} \left( f(Q(s',a')) + C^{(N)}(s',a')\right)
\\
    &\le f(r(s,a)) + \gamma \E_{s'} \max_{a'} f(Q(s',a')) + \gamma \E \max_{a'}C^{(N)}(s',a')
\\
    &= f(Q(s,a)) + \left[ f(r_i) + \gamma \E_{s'} V_f(s') - f(Q(s,a)) \right] + \gamma \E_{s'} \max_{a'}C^{(N)}(s',a')
\\
    &= f(Q(s,a)) +  C^{(N+1)}(s,a)
\end{align*}
\end{proof}

At this point, we verify that $C(s,a)>0$ which ensures the double-sided bounds above are valid.

To do so, we can simply bound the reward function $r_C(s,a)$. By determining $r_C(s,a)>0$, this will ensure $C(s,a) > \min r_C / (1-\gamma) > 0$.
\begin{align*}
    r_C(s,a) &= f(r(s,a)) + \gamma \mathbb{E}_{s'} V_f(s') - f(Q(s,a)) \\
    &\geq f(r(s,a)) + \gamma \mathbb{E}_{s'} f(V(s')) - f(Q(s,a)) \\
    &\geq f(r(s,a)) + f(\gamma \mathbb{E}_{s'} V(s')) - f(Q(s,a)) \\
    &\geq f(r(s,a) + \gamma \mathbb{E}_{s'} V(s')) - f(Q(s,a))\\
    &\geq 0
\end{align*}
where each line follows from the required conditions in Lemma \ref{thm:convex_cond_std}. A similar proof holds for showing the quantities $\hat{C}$, $D$, $\hat{D}$ are all non-negative.

We now prove the policy evaluation bound for standard RL.

\begin{lemma}
Consider the value of the policy $\pi_f(s) = \max_{a} f(Q(s,a))$ on the transformed task of interest, denoted by $\widetilde{Q}^{\pi_f}(s,a)$. 
 The sub-optimality of $\pi_f$ is then upper bounded by:
\begin{equation}
    \widetilde{Q}(s,a) - \widetilde{Q}^{\pi_f}(s,a) \leq D(s,a)
\end{equation}
where $D$ is the value of the policy $\pi_f$ in a task with reward
\begin{align*}
     r_D = \gamma \mathbb{E}_{s',a' \sim{} \pi_f} \biggr[ \max_{a} \big( f(Q(s',a')) + C(s',a') \big) - f(Q(s,a)) \biggr]
\end{align*}
\end{lemma}

\def\qpif{\widetilde{Q}^{\pi_f}(s,a)}
\def\qpifp{\widetilde{Q}^{\pi_f}(s',a')}

\def\fr{f\left(r(s,a)\right)}
\begin{proof}
    We will again prove this bound by induction on steps in the Bellman backup equation for the value of $\pi_f$, as given by the following fixed point equation:
    
    \begin{equation}
        \qpif = \fr + \gamma \E_{s', a' \sim{} \pi_f}\qpifp
    \end{equation}
    
    We consider the following initial conditions: $ \widetilde{Q}^{\pi_f(0)}(s,a) = \widetilde{Q}(s,a), D(s,a)=0$.
    We note that there is freedom in the choice of initial conditions, as the final statement (regarding the optimal value functions) holds regardless of initialization. As usual, the base case is trivially satisfied. We will now show that the equivalent inequality
    \begin{equation}
        \widetilde{Q}^{\pi_f (N)}(s,a) \ge \widetilde{Q}(s,a) - D^{(N)}(s,a)
    \end{equation}
    holds for all $N$. Similar to the previous proofs, we will subsequently take the limit $N \to \infty$ to recover the desired result.
    
    To do so, we consider the next step of the Bellman backup, and apply the inductive hypothesis:
    \begin{align}
        \widetilde{Q}^{\pi_f (N+1)}(s,a) &= \fr + \gamma \E_{s', a' \sim{} \pi_f}\left( \widetilde{Q}^{\pi_f (N)}(s',a') \right) \\
        &\geq \fr + \gamma \E_{s', a' \sim{} \pi_f}\left( \widetilde{Q}(s',a') - D^{(N)}(s',a') \right) \\
        &\geq \fr + \gamma \E_{s', a' \sim{} \pi_f}\left( f\left(Q(s',a')\right) - D^{(N)}(s',a') \right) \\
        &= \fr + \gamma \E_{s'} \widetilde{V}(s') + \gamma \E_{s', a' \sim{} \pi_f}\left( f\left(Q(s',a')\right) - D^{(N)}(s',a') - \widetilde{V}(s') \right) \\
        &\geq \widetilde{Q}(s,a) + \gamma \E_{s', a' \sim{} \pi_f}\left( f\left(Q(s',a')\right) - D^{(N)}(s',a') - \max_{a'} \left\{ f\left(Q(s',a')\right) + C(s',a') \right\} \right) \\
        &= \widetilde{Q}(s,a) - \gamma \E_{s', a' \sim{} \pi_f}\left( \max_{a'} \left\{ f\left(Q(s',a')\right) + C(s',a') \right\}  - f\left(Q(s',a')\right) + D^{(N)}(s',a')\right) \\
        &= \widetilde{Q}(s,a) - \left( r_D(s,a) + \gamma \E_{s',a' \sim{} \pi_f} D^{(N)}(s',a') \right) \\
        &= \widetilde{Q}(s,a) - D^{(N+1)}(s,a)
    \end{align}
    The third and fifth line follow from the previous bounds (Lemma 4.1). In the limit $N \to \infty$, we can thus see that the fixed point $D$ corresponds to the policy evaluation for $\pi_f$ in an environment with reward function $r_D$.
\end{proof}

Now we prove similar results, but for the ``concave conditions'' presented in the main text.
\begin{lemma}[Concave Conditions]\label{thm:concave_cond_std}
Given a primitive task with discount factor $\gamma$ and a bounded, continuous transformation function $f~:~X~\to~\mathbb{R}$ which satisfies:  
\begin{enumerate}
    \item $f$ is concave on its domain $X$ (for stochastic dynamics);
    \item $f$ is superlinear: 
    \begin{enumerate}[label=(\roman*)]
        \item $f(x+y) \geq f(x) + f(y)$ for all $x,y \in X$ 
        \item $f(\gamma x) \geq \gamma f(x)$ for all $x \in X$ 
    \end{enumerate}
    \item $f\left( \max_{a} \mathcal{Q}(s,a) \right) \geq \max_{a}~f\left( \mathcal{Q}(s,a) \right)$ for all functions $\mathcal{Q}:~\s~\times~\A \to X.$
\end{enumerate}
   
then the optimal action-value functions are now related in the following way:
\begin{equation}\label{eqn:concave_std}
    f(Q(s,a)) - \hat{C}(s,a) \leq \widetilde{Q}(s,a) \leq f(Q(s,a))
\end{equation}

where $\hat{C}$ is the optimal value function for a task with reward 
\begin{equation}
    \hat{r}_C(s,a) = f(Q(s,a)) - f(r(s,a)) - \gamma \E_{s'} V_f(s')
\end{equation}
\end{lemma}

\begin{proof}
The proof of $\widetilde{Q} \le f(Q)$ is the same as the preceding theorem's lower bound but with all inequalities reversed. To prove the upper bound involving $\hat{C}$, we use a similar approach
\begin{align*}
  \widetilde{Q}^{(N+1)}(s,a) &= f(r(s,a)) + \gamma \E_{s'} \max_{a'} \widetilde{Q}^{(N)}(s',a')
\\
    &\ge f(r(s,a)) + \gamma \E_{s'} \max_{a'} \left( f(Q(s',a')) - \hat{C}^{(N)}(s',a')\right)
\\
    &\ge f(r(s,a)) + \gamma \E_{s'} \left( \max_{a'} f(Q(s',a')) - \max_{a'} \hat{C}^{(N)}(s',a') \right)
\\
    &= f(Q(s,a)) - \left[f(Q(s,a)) - f(r(s,a)) - \gamma \E_{s'} V_f(s') + \gamma \E_{s'} \max_{a'} \hat{C}^{(N)}(s',a')\right]
\\
    &= f(Q(s,a)) - \hat{C}^{(N+1)}(s,a)
\end{align*}
The second line follows from the inductive hypothesis. The third line follows from the $\max$ of a difference. In the penultimate line, we add and subtract $f(Q)$, and identify the definitions for $V_f$ and the backup equation for $\hat{C}$. In the limit $N \to \infty$, we have the desired result.
\end{proof}

\begin{lemma}

Consider the value of the policy $\pi_f(s) = \max_{a} f(Q(s,a))$ on the transformed task of interest, denoted by $\widetilde{Q}^{\pi_f}(s,a)$. 
 The sub-optimality of $\pi_f$ is then upper bounded by:
\begin{equation}
    \widetilde{Q}(s,a) - \widetilde{Q}^{\pi_f}(s,a) \leq \hat{D}(s,a)
\end{equation}
where $\hat{D}$ is the value of the policy $\pi_f$ in a task with reward
\begin{equation}
     \hat{r}_D = \gamma \mathbb{E}_{s',a' \sim{} \pi_f} \biggr[ V_f(s') - f(Q(s',a')) + \hat{C}(s',a') \biggr]
\end{equation}
\end{lemma}

\begin{proof}
    The proof of this result is similar to that of Lemma 4.2, except now we must employ the corresponding results of Lemma 4.3. Beginning with a substitution of the inductive hypothesis:
    \begin{align}
        \widetilde{Q}^{\pi_f (N+1)}(s,a) &= \fr + \gamma \E_{s', a' \sim{} \pi_f}\left( \widetilde{Q}^{\pi_f (N)}(s',a') \right) \\
        &\geq \fr + \gamma \E_{s', a' \sim{} \pi_f}\left( \widetilde{Q}(s',a') - \hat{D}^{(N)}(s',a') \right) \\
        &\geq \fr + \gamma \E_{s', a' \sim{} \pi_f}\left( f\left(Q(s',a')\right) - \hat{C}(s',a') - \hat{D}^{(N)}(s',a') \right) \\
        &= \fr + \gamma \E_{s'} \widetilde{V}(s') + \gamma \E_{s', a' \sim{} \pi_f}\left( f\left(Q(s',a')\right) - \hat{C}(s',a') - \hat{D}^{(N)}(s',a') - \widetilde{V}(s') \right) \\
        &\geq \widetilde{Q}(s,a) + \gamma \E_{s', a' \sim{} \pi_f}\left( f\left(Q(s',a')\right) - \hat{C}(s',a') - \hat{D}^{(N)}(s',a') - V_f(s') \right) \\
        &= \widetilde{Q}(s,a) - \gamma \E_{s', a' \sim{} \pi_f}\left( V_f(s') - f\left(Q(s',a')\right)+ \hat{C}(s',a') + \hat{D}^{(N)}(s',a')\right) \\
        &= \widetilde{Q}(s,a) - \left( \hat{r}_D(s,a) + \gamma \E_{s',a' \sim{} \pi_f} \hat{D}^{(N)}(s',a') \right) \\
        &= \widetilde{Q}(s,a) - \hat{D}^{(N+1)}(s,a)
    \end{align}
    
\end{proof}

Now we provide further details on the technical conditions for compositions (rather than transformations) of primitive tasks to satisfy the derived bounds. 
\begin{lemma}[Convex Composition of Primitive Tasks]
Suppose $F:\bigotimes_k X^{(k)} \to \mathbb{R}$ is convex on its domain and is sublinear (separately in each argument), that is:

\begin{align}
    F(x^{(1)}+y^{(1)},x^{(2)}, \dotsc, x^{(M)}) &\leq F(x^{(1)},x^{(2)},\dotsc, x^{(M)}) + F(y^{(1)},x^{(2)},\dotsc, x^{(M)}) \\
    F(x^{(1)},x^{(2)}+y^{(2)}, \dotsc, x^{(M)}) &\leq F(x^{(1)},y^{(2)},\dotsc, x^{(M)}) + F(x^{(1)},y^{(2)},\dotsc, x^{(M)})
\end{align}
and similarly for the remaining arguments.

\begin{equation}
    F(\gamma x^{(1)}, \dotsc, \gamma  x^{(M)}) \leq \gamma F(x^{(1)}, \dotsc x^{(M)})
\end{equation}and also satisfies
\begin{equation}
    F\left( \max_a  \Q^{(1)}(s,a), \dotsc , \max_a \Q^{(M)}(s,a) \right) \leq \max_a F\left(\Q^{(1)}(s,a), \dotsc , \Q^{(M)}(s,a)\right)
\end{equation}
for all functions $\Q^{(k)}:\s \times \A \to \mathbb{R}$.
Then, 
\begin{equation}
    F(\vec{Q}(s,a)) \le \widetilde{Q}(s,a) \le F(\vec{Q}(s,a)) + C(s,a)
\end{equation}
where we use a vector notation to emphasize that the function acts over the set of optimal value functions $\{Q^{(k)}\}$  corresponding to each primitive task, defined by $r^{(k)}$.
\end{lemma}
\begin{proof}
    The proof of this statement is identical to the proof of Lemma \ref{thm:convex_cond_std}, now using the fact that $F$ is a multivariable function $F: X^N \to Y$, with each argument obeying the required conditions. $C$ takes the analogous definition as provided for the original result.
\end{proof}

\begin{lemma}[Concave Composition of Primitive Tasks]\label{thm:compos_concave_std}
If on the other hand $F$ is concave and superlinear in each argument, and also satisfies

\begin{equation}
    F\left( \max_a  \Q^{(1)}(s,a), \dotsc , \max_a \Q^{(M)}(s,a) \right) \leq \max_a F\left(\Q^{(1)}(s,a), \dotsc , \Q^{(M)}(s,a)\right)
\end{equation}
for all functions $\Q^{(k)}:\s \times \A \to \mathbb{R}$, then

\begin{equation}
    F(\vec{Q}(s,a)) - \hat{C}(s,a) \le \widetilde{Q}(s,a) \le F(\vec{Q}(s,a)).
\end{equation}
\end{lemma}
\begin{proof}
    Again, the proof of this statement is identical to the proof of Lemma \ref{thm:concave_cond_std}, now using the fact that $F$ is a multivariable function $F: X^N \to Y$, with each argument obeying the required conditions.
\end{proof}

\subsection{Examples of transformations and compositions}
In this section, we consider the examples of transformations and compositions mentioned in the main text, and discuss the corresponding results in standard RL.
\begin{remark}
Given the convex composition of subtasks  $r^{(c)} \equiv f(\{r^{(k)}\}) = \sum_k \alpha_k r^{(k)}$ considered by \cite{Haarnoja2018} and \cite{hunt_diverg}, we can use the results of Lemma \ref{thm:compos_concave_std} to bound the optimal $Q$ function by using the optimal $Q$ functions for the primitive tasks:
\begin{equation}
    Q^{(c)}(s,a) \leq \sum_k \alpha_k Q^{(k)}(s,a)
\end{equation}
\end{remark}
\begin{proof}
    In standard RL, we need only show that $f( \max_i  x_{1i}, \dotsc , \max_i x_{ni} ) \geq \max_i f(x_i, \dotsc , x_n)$:
    \begin{equation}
        \sum_k \alpha_k  \max_i  x^{(k)}_{i} \geq \max_i \sum_k \alpha_k x^{(k)}_i
    \end{equation} which holds given $\alpha_k \geq 0$ for all $k$.
    We also note that in this case the result clearly holds for general $\alpha_k \geq 0$ not necessarily with $\sum_k \alpha_k = 1$ (as assumed in \cite{Haarnoja2018} and \cite{hunt_diverg}).
\end{proof}

\begin{remark}
Given the AND composition defined above and considered in \cite{boolean}, we have the following result in standard RL:
\begin{equation}
    Q^{\text{AND}}(s,a) \leq \min_k \left\{Q^{(k)}(s,a)\right\}
\end{equation}
\end{remark}
\begin{proof}
We could proceed via induction as in the previous proofs, or simply use the above remark, and prove the necessary conditions on the function $f(\cdot) = \min(\cdot)$.
The function $\min(\cdot)$ is concave in each argument. It is also straightforward to show that $\min(\cdot)$ is subadditive over all arguments.
\end{proof}

\begin{remark}
Result of (hard) OR composition result in standard RL:
\begin{equation}
    Q^{\text{OR}}(s,a) \geq \max_k \left\{Q^{(k)}(s,a)\right\}
\end{equation}
\end{remark}

\begin{proof}
The proof is analogous to the (hard) AND result: $\max$ is a convex, superadditive function.

\end{proof}

\begin{remark}
Result for NOT operation in standard RL: 
\begin{equation}
    Q^{\text{NOT}}(s,a) \geq - Q(s,a)
\end{equation}
\end{remark}

\begin{proof}
    Since the ``NOT'' gate is a unary function, and we are in the standard RL setting, we must check the conditions of Lemma 4.1 or 4.3. Moreoever, since the transformation function applied to the rewards, $f(r)=-r$ is linear, we must check the final condition: $\max_i\{-x_i\} = -\min_i\{x_i\} \geq -\max_i\{x_i\}$. This is the condition required by the concave conditions.
\end{proof}

\section{Proofs for Entropy-Regularized RL}

Let $X$ be the codomain for the $Q$ function of the primitive task ($Q: \s \times \A \to X \subseteq \mathbb{R}$).
\begin{lemma}[Convex Conditions]
\label{thm:forward_cond_entropy-regularized}
Given a bounded, continuous transformation function $f~:~X~\to~\mathbb{R}$ which satisfies:  
\begin{enumerate}
    \item $f$ is convex on its domain $X$ (for stochastic dynamics);
    \item $f$ is sublinear:
       \begin{enumerate}[label=(\roman*)]
        \item $f(x+y) \leq f(x) + f(y)$ for all $x,y \in X$
        \item $f(\gamma x) \leq \gamma f(x)$ for all $x \in X$ 
    \end{enumerate}
    \item $f\left( \log \E \exp \mathcal{Q}(s,a) \right) \leq \log \E \exp f\left( \mathcal{Q}(s,a) \right)$ for all functions $\mathcal{Q}:~\s~\times~\A \to \mathbb{R}.$
    
\end{enumerate} 

then the optimal action-value function for the transformed rewards, $\widetilde{Q}$, is now related to the optimal action-value function with respect to the original rewards by:
\begin{equation}\label{eq:convex_entropy-regularized}
   f \left( Q(s,a) \right) \leq \widetilde{Q}(s,a) \leq f \left( Q(s,a) \right) + C(s,a)
\end{equation}

\end{lemma}

\begin{proof}
We will again prove the result with induction, beginning by writing the backup equation for the optimal soft $Q$ function in the transformed reward environment to prove the upper bound on $\widetilde{Q}$:
\begin{equation}
    \widetilde{Q}^{(N+1)}(s,a) = f(r(s,a)) + \gamma \mathbb{E}_{s' \sim{} p(s'|s,a)} \frac{1}{\beta} \log \mathbb{E}_{a' \sim{} \pi_0(a'|s')} \exp \left(\beta Q^{(N)}(s',a')\right)
\end{equation}
where $p$ is the dynamics and $\pi_0$ is the prior policy. Applying the inductive assumption,
\begin{equation}
    \widetilde{Q}^{(N+1)}(s,a) \geq f(r(s,a)) + \gamma \mathbb{E}_{s' \sim{} p(s'|s,a)} \frac{1}{\beta} \log \mathbb{E}_{a' \sim{} \pi_0(a'|s')}\exp \left( f\left(\beta Q^{(N)}(s',a')\right) \right)
\end{equation}
Next, using the third condition on $f$ as well as its convexity, we may factor $f$ out of the expectations by Jensen's inequality:
\begin{equation}
    \widetilde{Q}^{(N+1)}(s,a) \geq f(r(s,a)) + \gamma f\left( \mathbb{E}_{s' \sim{} p(s'|s,a)} \frac{1}{\beta} \log \mathbb{E}_{a' \sim{} \pi_0(a'|s')}\exp \left(\beta Q^{(N)}(s',a')\right) \right)
\end{equation}
Finally, using the sublinearity conditions of $f$, we arrive at
\begin{equation}
    \widetilde{Q}^{(N+1)}(s,a) \geq f \left(r(s,a) + \gamma \mathbb{E}_{s' \sim{} p(s'|s,a)} \frac{1}{\beta} \log \mathbb{E}_{a' \sim{} \pi_0(a'|s')}\exp \left(\beta Q^{(N)}(s',a')\right) \right) 
\end{equation}
The right hand side is $f \left(Q^{(N+1)}(s,a) \right)$. In the limit $N\to \infty,\ Q^{(N)}(s,a)  \to Q(s,a)$ so the inductive proof for the upper bound is complete.

    Let $f$ satisfy the ``convex conditions''. 
    Consider the backup equation for $\widetilde{Q}$. For the initialization (base case) we let $\widetilde{Q}^{(0)}(s,a)=f\left(Q(s,a)\right)$ and $C^{(0)}(s,a)=0$.
    Using the inductive assumption,

\begin{align*}
    \widetilde{Q}^{(N+1)}(s,a) &= f(r(s,a)) + \frac{\gamma}{\beta} \E_{s' \sim{} p} \log \E_{a' \sim{} \pi_0} \exp \beta \widetilde{Q}^{(N)}(s',a')
\\
    &\leq f(r(s,a)) + \frac{\gamma}{\beta} \E_{s'} \log \E_{a'} \exp \beta \left( f(Q(s',a')) + C^{(N)}(s',a')\right)
\\
    &\leq f(r(s,a)) + \frac{\gamma}{\beta} \E_{s'} \left(\log \E_{a'} \exp \beta  f(Q(s',a')) + \max_{a'} C^{(N)}(s',a')\right)
\\
    &= f(Q(s,a)) + f(r(s,a)) + \gamma \E_{s'} V_f(s') - f(Q(s,a)) + \gamma \E_{s'} \max_{a'} C^{(N)}(s',a')
\\
    &= f(Q(s,a)) + C^{(N+1)}(s,a)
\end{align*}

Therefore in the limit $N \to \infty$, we have:
$\widetilde{Q}(s,a) \leq f(Q(s,a)) + C(s,a)$ as desired. We note that since $f(r(s,a))~+~\gamma~\E_{s'}~V_f(s')~\geq~f(Q(s,a))$, we immediately have $C(s,a) \ge 0$, as is required for the bound to be non-vacuous.
\end{proof}

\begin{lemma}
Consider the soft value of the policy $\pi_f \propto \exp \beta f(Q)$ on the transformed task of interest, denoted by $\widetilde{Q}^{\pi_f}$(s,a). 
 The sub-optimality of $\pi_f$ is then upper bounded by:
\begin{equation}
    \widetilde{Q}(s,a) - \widetilde{Q}^{\pi_f}(s,a) \leq D(s,a)
\end{equation}
where $D$ is the value of the policy $\pi_f$
 with reward
\begin{equation}
    r_D(s,a) = \gamma \mathbb{E}_{s' \sim{} p} 
\left[ \max_{b} \left\{ f\left(Q(s',b)\right) + C(s',b) \right\} -V_f(s') \right]
\end{equation}

\end{lemma}

\begin{proof}
To prove the (soft) policy evaluation bound, we use iterations of soft-policy evaluation \cite{Haarnoja_SAC} and denote iteration $N$ of the evaluation of $\pi_f$ in the composite environment as $\widetilde{Q}^{\pi_f(N)}$. Beginning with the definitions $\widetilde{Q}^{\pi_f(0)}(s,a) = Q(s,a)$ (since the evaluation is independent of the initialization), and $D^{(0)}=0$, the $N=0$ step is trivially satisfied. Assuming the inductive hypothesis, we consider the next step of soft policy evaluation:
As in the previous policy evaluation results, we prove an equivalent result with induction.

\begin{align*}
    \widetilde{Q}^{\pi_f(N+1)}(s,a) &= f(r(s,a)) + \gamma \E_{s' \sim{} p}\E_{ a'\sim{} \pi_f} \left[\widetilde{Q}^{\pi_f(N)}(s',a') - \frac{1}{\beta} \log \frac{\pi_f(a'|s')}{\pi_0(a'|s')} \right]
\\
    &\geq f(r(s,a)) + \gamma \E_{s',a'} \left[\widetilde{Q}(s',a') - D^{(N)}(s',a') - f(Q(s',a')) + V_f(s') \right]
\\
    &= f(r(s,a)) + \gamma \E_{s'}\widetilde{V}(s')  + \gamma \E_{s',a'} \left[\widetilde{Q}(s',a') - D^{(N)}(s',a') - f(Q(s',a'))  + V_f(s') - \widetilde{V}(s')\right] \\
    &\geq \widetilde{Q}(s,a)  + \gamma \E_{s',a'} \left[f(Q(s',a')) - D^{(N)}(s',a') - f(Q(s',a'))  + V_f(s') - \widetilde{V}(s')\right] \\
    &\geq \widetilde{Q}(s,a)  + \gamma \E_{s',a'} \left[ - D^{(N)}(s',a') + V_f(s') - \max_{b} \left\{ f\left(Q(s',b)\right) + C(s',b) \right\} \right] \\
    &\geq  \widetilde{Q}(s,a) -D^{(N+1)}(s,a)
\\ 
\end{align*}

where we have used $\widetilde{Q}(s,a) \geq f(Q(s,a))$ in the fourth line.

where we have used the fact that $ \widetilde{V}(s) \leq \max_b \left\{ f\left(Q(s,b)\right) + \max_a C(s,a)\right\}$ and $\widetilde{Q}(s,a) - f(Q(s,a)) \geq 0$ which both follow from the previously stated bounds.
\end{proof}

\begin{lemma}[Concave Conditions]
\label{thm:reverse_cond_entropy-regularized}
Given a bounded, continuous transformation function $f~:~X~\to~\mathbb{R}$ which satisfies:  
\begin{enumerate}
    \item $f$ is concave on its domain $X$ (for stochastic dynamics); 
    \item $f$ is superlinear:
     \begin{enumerate}[label=(\roman*)]
        \item $f(x+y) \geq f(x) + f(y)$ for all $x,y \in X$ 
        \item $f(\gamma x) \geq \gamma f(x)$ for all $x \in X$ 
    \end{enumerate}
    \item $f\left( \log \E \exp \mathcal{Q}(s,a) \right) \geq \log \E \exp f\left( \mathcal{Q}(s,a) \right)$ for all functions $\mathcal{Q}:~\s~\times~\A \to \mathbb{R}.$

\end{enumerate}
then the optimal action-value function for the transformed rewards obeys the following inequality:
\begin{equation}\label{eq:concave_entropy-regularized}
    f\left( Q(s,a) \right) - \hat{C}(s,a) \leq \widetilde{Q}(s,a) \leq f \left( Q(s,a) \right)
\end{equation}
\end{lemma}

\begin{proof}
The proof of the upper bound is the same as the preceding theorem's lower bound with all inequalities reversed.
For the lower bound involving $C$,

    Again consider the backup equation for $\widetilde{Q}$.
    Using the definitions and inductive assumption as before, we have

\begin{align*}
    \widetilde{Q}^{(N+1)}(s,a) &= f(r(s,a)) + \frac{\gamma}{\beta} \E_{s' \sim{} p} \log \E_{a' \sim{} \pi_0} \exp \beta \widetilde{Q}^{(N)}(s',a')
\\
    &\geq f(r(s,a)) + \frac{\gamma}{\beta} \E_{s'} \log \E_{a'} \exp \beta \left( f(Q(s',a')) -\hat{C}^{(N)}(s',a')\right)
\\
    &\geq f(r(s,a)) + \frac{\gamma}{\beta} \E_{s'} \left(\log \E_{a'} \exp \beta  f(Q(s',a')) - \max_{a'} \hat{C}^{(N)}(s',a')\right)
\\
    &= f(Q(s,a)) - \left[f(Q(s,a)) - f(r(s,a)) - \gamma \E_{s'} V_f(s') + \gamma \E_{s'} \max_{a'} \hat{C}^{(N)}(s',a')\right]
\\
    &= f(Q(s,a)) -  \hat{C}^{(N+1)}(s,a)
\end{align*}

Therefore in the limit $N \to \infty$, we have:
$\widetilde{Q}(s,a) \geq f(Q(s,a)) - \hat{C}(s,a)$ as desired.
\end{proof}

\begin{lemma}
Consider the soft value of the policy $\pi_f \propto \exp \beta f(Q)$ on the transformed task of interest, denoted by $\widetilde{Q}^{\pi_f}$(s,a). 
 The sub-optimality of $\pi_f$ is then upper bounded by:
\begin{equation}
    \widetilde{Q}(s,a) - \widetilde{Q}^{\pi_f}(s,a) \leq \hat{D}(s,a)
\end{equation}
where $\hat{D}$ is the fixed point of 
\begin{equation}
    \hat{D}(s,a) \xleftarrow{} \gamma \mathbb{E}_{s' \sim{} p}\mathbb{E}_{a' \sim{} \pi_f} \left[ \hat{C}(s',a') + \hat{D}(s',a') \right]
\end{equation}

\end{lemma}

\begin{proof}
We will show the policy evaluation result by induction, by evaluating $\pi_f \propto \exp(\beta f(Q))$ in the environment with rewards $f(r)$. We shall denote iterations of policy evaluation for $\pi_f$ in the environment with rewards $f(r)$ by $\widetilde{Q}^{\pi_f(N)}(s,a)$.

\begin{align*}
    \widetilde{Q}^{\pi_f(N+1)}(s,a) &= f(r(s,a)) + \gamma \E_{s'\sim{}p} \E_{a'\sim{} \pi_f} \left[\widetilde{Q}^{\pi_f(N)}(s',a') - \frac{1}{\beta} \log \frac{\pi_f(a'|s')}{\pi_0(a'|s')} \right]
\\
    &\geq f(r(s,a)) + \gamma \E_{s',a'} \left[\widetilde{Q}(s',a')-\hat{D}^{(N)}(s',a') - (f(Q(s',a')) - V_f(s')) \right]
\\
    &\geq f(r(s,a)) + \gamma \E_{s',a'} \left[\widetilde{Q}(s',a')-\hat{D}^{(N)}(s',a') - \widetilde{Q}(s',a') -\hat{C}(s',a') + V_f(s') \right]
\\
    &\ge f(r(s,a)) + \gamma \E_{s'} \widetilde{V}(s') - \gamma \E_{s',a'} \left[\hat{D}^{(N)}(s',a') + \hat{C}(s',a') \right]
\\
    &= \widetilde{Q}(s,a)) - \hat{D}^{(N+1)}(s,a)
\\
\end{align*}
where we have used the inductive assumption and $V_f(s) \ge \widetilde{V}(s)$ and which follows from the previously stated bounds.
Therefore in the limit $N \to \infty$, we have:
$    \widetilde{Q}^{\pi_f}(s,a) \geq \widetilde{Q}(s,a) - \hat{D}(s,a)
$ as desired.
\end{proof}

\begin{lemma}[Convex Composition of Primitive Tasks]\label{thm:compos_convex_maxent}
Suppose $F:X^N \to Y$ is convex on its domain $X^N$ and satisfies all conditions of Lemma 5.1 (Main Text) component-wise. Then, 
\begin{equation}
    F(\vec{Q}(s,a)) \le \widetilde{Q}(s,a) \le F(\vec{Q}(s,a)) + C(s,a)
\end{equation}
and 
\begin{equation}
    \widetilde{Q}^{\pi_f}(s,a) \geq \widetilde{Q}(s,a) - D(s,a)
\end{equation}
where we use a vector notation to emphasize that the function acts over the set of optimal $\{Q_k\}$ functions corresponding to each subtask, defined by $r_k$.
\end{lemma}
\begin{proof}
    The proof of this statement is identical to the previous proofs, now using the fact that $F$ is a multivariable function $F: X^N \to Y$, with each argument obeying the required conditions.
\end{proof}

\begin{lemma}[Concave Composition of Primitive Tasks]\label{thm:compos_concave_maxent}
If on the other hand $F$ is concave and and satisfies all conditions of Lemma 5.2 (Main Text) component-wise, then
\begin{equation}
    F(\vec{Q}(s,a)) - \hat{C}(s,a) \le \widetilde{Q}(s,a) \le F(\vec{Q}(s,a)).
\end{equation}
and 
\begin{equation}
    \widetilde{Q}^{\pi_f}(s,a) \geq \widetilde{Q}(s,a) - \hat{D}(s,a)
\end{equation}
\end{lemma}
\begin{proof}
    Again, the proof of this statement is identical to the previous proofs, now using the fact that $F$ is a multivariable function $F: X^N \to Y$, with each argument obeying the required conditions.
\end{proof}

\subsection{Examples of Transformations and Compositions}
In this section we consider several examples mentioned in the main text, and show how they are proved with our results in entropy-regularized RL.

\begin{remark}
Given the convex composition of subtasks  $r^{(c)} \equiv F(\{r^{(k)}\}) = \sum_k \alpha_k r^{(k)}$ considered by \cite{Haarnoja2018} and \cite{hunt_diverg}, we can use the results of Lemma \ref{thm:compos_concave_maxent} to bound the optimal $Q$ function by using the optimal $Q$ functions for the primitive tasks:
\begin{equation}
    Q^{(c)}(s,a) \leq \sum_k \alpha_k Q^{(k)}(s,a)
\end{equation}
\end{remark}
\begin{proof}
    In entropy-regularized RL we need to show that the final condition holds (in vectorized form). This is simply H\"older's inequality \cite{hardy1952inequalities} for vector-valued functions in a probability space (with measure defined by $\pi_0$).
    
\end{proof}

\begin{remark}
Given the AND composition defined above and considered in \cite{boolean}, we have the following result in standard RL:
\begin{equation}
    Q^{\text{AND}}(s,a) \leq \min_k \left\{Q^{(k)}(s,a)\right\}
\end{equation}
\end{remark}
\begin{proof}
The function $\min(\cdot)$ is concave in each argument. It is also straightforward to show that $\min(\cdot)$ is subadditive over all arguments. For the final condition, the $\min$ acts globally over all subtasks:
\begin{equation}
    \min_k \left\{ \frac{1}{\beta}\log\mathbb{E}_{a \sim{} \pi_0(a|s)} \exp\left(\beta \Q^{(k)}(s,a)\right)\right\} \leq \frac{1}{\beta}\log\mathbb{E}_{a \sim{} \pi_0(a|s)} \exp\left( \beta \min_k \left\{\Q^{(k)}(s,a)\right\}\right).
\end{equation}
\end{proof}

\begin{remark}
Result of (hard) OR composition result in standard RL:
\begin{equation}
    Q^{\text{OR}}(s,a) \geq \max_k \left\{Q^{(k)}(s,a)\right\}
\end{equation}
\end{remark}

\begin{proof}
The proof is analogous to the (hard) AND result: $\max$ is a convex, superadditive function.
For the final condition, the $\max$ again acts globally over all subtasks:
\begin{equation}
    \max_k \left\{ \frac{1}{\beta}\log\mathbb{E}_{a \sim{} \pi_0(a|s)} \exp\left(\beta 
    \Q^{(k)}(s,a)\right)\right\} \ge \frac{1}{\beta}\log\mathbb{E}_{a \sim{} \pi_0(a|s)} \exp\left( \beta \max_k \left\{\Q^{(k)}(s,a)\right\}\right).
\end{equation}

\end{proof}

\begin{remark}
Again we consider the NOT operation defined above, now in entropy-regularized RL, which yields the bound:
\begin{equation}
    Q^{\text{NOT}}(s,a) \geq - Q(s,a)
\end{equation}
\end{remark}

\begin{proof}
As in the standard RL case, we need only consider the third condition of either Lemma 5.1 or 5.3.
In particular, we show
\begin{equation}
f\left( \log \E \exp \mathcal{Q}(s,a) \right) \leq \log \E \exp f\left( \mathcal{Q}(s,a) \right)
\end{equation}
for all functions $\mathcal{Q}:~\s~\times~\A \to \mathbb{R}$. This follows from
\begin{align}
\frac{1}{\E \exp \mathcal{Q}(s,a)} \leq \E \frac{1}{\exp \mathcal{Q}(s,a) }
\end{align}
which is given by Jensen's inequality, since the function $f(x)=1/x$ is convex.

\end{proof}

\begin{remark}[Linear Scaling]
\label{thm:scaling}
Given some $k \in (0,1)$ the function $f(x) = k x$ satisfies the results of the first theorem. Conversely, if  $k \geq 1$, $f(x) = k x$ satisfies the results of the second theorem. 
\end{remark}
\begin{proof}
    This result (specifically the third condition of Lemma 5.1, 5.3) follows from the monotonicity of $\ell_p$ norms.
\end{proof}

Since we have already shown the case of $k=-1$ (NOT gate), with the result of Theorem \ref{thm:compos}, the case for all $k \in \mathbb{R}$ has been characterized.

\section{Extension for Error-Prone $Q$-Values}
In this section, we provide some discussion on the case of inexact $Q$-values, as often occurs in practice (discussed at the end of Section 4.1 in the main text). We focus on the case of task transformation in standard RL. The corresponding statements in the settings of composition and entropy-regularized RL follow similarly.

As our starting point, we assume that an ``$\varepsilon$-optimal estimate'' $\overbar{Q}(s,a)$ for a primitive task's exact value function $Q(s,a)$ is known.

\begin{definition}
An $\varepsilon$-optimal $Q$-function, $\overbar{Q}$, satisfies 
\begin{equation}
    |Q(s,a)-\overbar{Q}(s,a)|\leq \varepsilon
\end{equation}
for all $s \in \s, a \in \A$.
\end{definition}

To allow the derived double-sided bounds on the transformed tasks' $Q$-values to carry over to this more general setting, we assume that the transformation function is $L$-Lipschitz continuous. With these assumptions, we prove the following extensions of Lemma 4.1 and 4.3:

\begin{customlemma}{4.1A}[Convex Conditions, Error-Prone]\label{thm:convex_cond_std_err}
Given a primitive task with discount factor $\gamma$, corresponding $\varepsilon$-optimal value function $\overbar{Q}$, and a bounded, continuous, $L$-Lipschitz transformation function $f~:~X~\to~\mathbb{R}$ which satisfies:  
\begin{enumerate}
    \item $f$ is convex on its domain $X$ (for stochastic dynamics);
    \item $f$ is sublinear:
    \begin{enumerate}[label=(\roman*)]
        \item $f(x+y) \leq f(x) + f(y)$ for all $x,y \in X$
        \item $f(\gamma x) \leq \gamma f(x)$ for all $x \in X$ 
    \end{enumerate}
    \item $f\left( \max_{a} \mathcal{Q}(s,a) \right) \leq \max_{a}~f\left( \mathcal{Q}(s,a) \right)$ for all $\mathcal{Q}: \s \times \A \to \mathbb{R}.$
\end{enumerate}

then the optimal action-value function for the transformed rewards, $\widetilde{Q}$, is now related to the optimal action-value function with respect to the original rewards  by:

\begin{equation}\label{eqn:convex_std_err}
    f(\overbar{Q}(s,a)) - L \varepsilon \leq \widetilde{Q}(s,a) \leq f(\overbar{Q}(s,a)) + \overbar{C}(s,a)  + \frac{2}{1-\gamma}L \varepsilon
\end{equation}

where $\overbar{C}$ is the optimal value function for a task with reward
\begin{equation}\label{eq:std_convex_C_def_err}
    \overbar{r_C}(s,a) = f(r(s,a)) + \gamma \mathbb{E}_{s'} \overbar{V_f}(s') - f(\overbar{Q}(s,a)).
\end{equation}
with $\overbar{V_f}(s)=\max_a f(\overbar{Q}(s,a))$.
\end{customlemma}

Note that as $\varepsilon \to 0$, the exact result (Lemma 4.1) is recovered. If the function $\overbar{C}$ is not known exactly, one can similarly exchange $\overbar{C}$ for $\overbar{\overbar{C}}$, an $\varepsilon$-optimal estimate for $\overbar{C}$. This consideration loosens the upper-bound by an addition of $\varepsilon$, shown at the end of the proof.

We will make use a well-known result (cf. proof of Lemma 1 in \cite{barreto_sf}) that bounds the difference in optimal $Q$-values for two tasks with different reward functions.
\begin{lemma}
    Let two tasks, only differing in their reward functions, be given with reward $r_1(s,a)$ and $r_2(s,a)$, respectively. Suppose $|r_1(s,a)-r_2(s,a)|\leq \delta$ Then, the optimal value functions for the tasks satisfies:
    \begin{equation}
        |Q_1(s,a)-Q_2(s,a)|\leq \frac{\delta}{1-\gamma}
    \end{equation}
    \label{lem:bounded_q_diff}
\end{lemma}

Now we are in a position to prove Lemma \ref{thm:convex_cond_std_err}:
\begin{proof}
To prove the lower bound, we begin with the original lower bound in Lemma 4.1, for the optimal primitive task $Q$-values:
\begin{equation}
    \widetilde{Q}(s,a) \geq f(Q(s,a)),
\end{equation}
or equivalently
\begin{align}
-\widetilde{Q}(s,a) &\leq -f(Q(s,a)) \\
-\widetilde{Q}(s,a) &\leq -f(Q(s,a)) + f(\overbar{Q}(s,a))  -f(\overbar{Q}(s,a))  \\
-\widetilde{Q}(s,a) &\leq |f(Q(s,a)) - f(\overbar{Q}(s,a))| - f(\overbar{Q}(s,a)) \\
\widetilde{Q}(s,a) &\geq -|f(Q(s,a)) - f(\overbar{Q}(s,a))| + f(\overbar{Q}(s,a)) \\
\widetilde{Q}(s,a) &\geq -L|Q(s,a) - \overbar{Q}(s,a)| + f(\overbar{Q}(s,a))  \\
\widetilde{Q}(s,a) &\geq f(\overbar{Q}(s,a)) - L \varepsilon  \\
\end{align}

Where the final steps follow from the function $f$ being $L$-Lipschitz and the definition of $\varepsilon$-optimality of $\overbar{Q}(s,a)$.

To prove the upper bound, we take a similar approach, noting that the reward function $r_C$ in Lemma 4.1 must be updated to account for the inexact $Q$-values. Therefore, we must account for the following error propagations: 
\begin{align*}
    Q(s,a) &\to \overbar{Q}(s,a) \\
    V_f(s) &\to \overbar{V_f}(s)\\
    r_C(s,a) &\to \overbar{r_C}(s,a).
\end{align*}
We first find the difference between $r_C$ and $\overbar{r_C}$ to be bounded by $(1+\gamma)L\varepsilon$:
\begin{align}
    |r_C(s,a)-\overbar{r_C}(s,a)| &= |\gamma \E_{s' \sim{} p } V_f^*(s') - f(Q^*(s,a)) - \gamma \E_{s' \sim{} p } V_f(s') + f(Q(s,a))| \\
    &\leq \gamma \E_{s'} |V_f^*(s') - V_f(s')| + |f(Q^*(s,a)) - f(Q(s,a))| \\
    &\leq \gamma \E_{s'} \max_{a'} |f(Q^*(s',a')) - f(Q(s',a'))| + |f(Q^*(s,a)) - f(Q(s,a))| \\
    &\leq (1+\gamma)L \varepsilon
\end{align}
where in the third line we have used the bound $|\max_{x} f(x) - \max_{x} g(x)| \leq \max_{x} |f(x)-g(x)|$.

Now, applying Lemma \ref{lem:bounded_q_diff} to the reward functions $r_C$ and $\overbar{r_C}$:
\begin{equation}
    |C(s,a) - \overbar{C}(s,a)| \leq \frac{ (1+\gamma)}{1-\gamma}L \varepsilon
\end{equation}
With the same technique as was used above for the lower bound, we find:
\begin{align}
    \widetilde{Q}(s,a) &\leq f(Q(s,a)) +  C(s,a) \\
    &\leq f(\overbar{Q}(s,a)) + L \varepsilon + C(s,a) \\
    &= f(\overbar{Q}(s,a)) + L \varepsilon + \overbar{C}(s,a) - \overbar{C}(s,a) + C(s,a) \\
    &\leq f(\overbar{Q}(s,a)) +  L \varepsilon + |C(s,a) - \overbar{C}(s,a)| + \overbar{C}(s,a) \\ 
    &\leq f(\overbar{Q}(s,a)) + L \varepsilon + \overbar{C}(s,a) + \frac{ (1+\gamma)}{1-\gamma}L \varepsilon \\ 
    &= f(\overbar{Q}(s,a)) + \overbar{C}(s,a) + \frac{2}{1-\gamma}L \varepsilon 
\end{align}
Further extending the result to the case where only an $\varepsilon$-optimal estimate of $\overbar{C}$ is known, denoted by $\overbar{\overbar{C}}$, we find:
\begin{align}
    \widetilde{Q}(s,a) &\leq f(\overbar{Q}(s,a)) + \overbar{C}(s,a) + \frac{2}{1-\gamma}L \varepsilon \\ 
    &\leq f(\overbar{Q}(s,a)) + \overbar{\overbar{C}}(s,a) + |\overbar{\overbar{C}}(s,a)- \overbar{C}(s,a)| + \frac{2}{1-\gamma}L \varepsilon \\ 
    &\leq f(\overbar{Q}(s,a)) + \overbar{\overbar{C}}(s,a) + \varepsilon + \frac{2}{1-\gamma}L \varepsilon \\ 
    &= f(\overbar{Q}(s,a)) + \overbar{\overbar{C}}(s,a) + \left(1+ \frac{2}{1-\gamma}L \right)\varepsilon
\end{align}
\end{proof}

Similarly, Lemma 4.3 from the main text can be extended under the same conditions:
\begin{customlemma}{4.3A}[Concave Conditions, Error-Prone]\label{thm:concave_cond_std_err}
Given a primitive task with discount factor $\gamma$, corresponding $\varepsilon$-optimal value function $\overbar{Q}$, and a bounded, continuous, $L$-Lipschitz transformation function $f~:~X~\to~\mathbb{R}$ which satisfies:  
\begin{enumerate}
    \item $f$ is concave on its domain $X$ (for stochastic dynamics);
    \item $f$ is superlinear: 
    \begin{enumerate}[label=(\roman*)]
        \item $f(x+y) \geq f(x) + f(y)$ for all $x,y \in X$ 
        \item $f(\gamma x) \geq \gamma f(x)$ for all $x \in X$ 
    \end{enumerate}
    \item $f\left( \max_{a} \mathcal{Q}(s,a) \right) \geq \max_{a}~f\left( \mathcal{Q}(s,a) \right)$ for all functions $\mathcal{Q}:~\s~\times~\A \to X.$
\end{enumerate}
   
then the optimal action-value functions are now related in the following way:
\begin{equation}\label{eqn:concave_std_err}
    f(\overbar{Q}(s,a)) - \overbar{\hat{C}}(s,a)-\frac{2}{1-\gamma}L \varepsilon \leq \widetilde{Q}(s,a) \leq f(\overbar{Q}(s,a)) + L \varepsilon
\end{equation}

where $\overbar{\hat{C}}$ is the optimal value function for a task with reward 
\begin{equation}
    \overbar{\hat{r}_C}(s,a) = f(\overbar{Q}(s,a)) - f(r(s,a)) - \gamma \E_{s'\sim{}p} \overbar{V_f}(s')
\end{equation}
with $\overbar{V_f}(s)=\max_a f(\overbar{Q}(s,a))$.
\end{customlemma}

The proof of Lemma \ref{thm:concave_cond_std_err} is the same as that given above for Lemma \ref{thm:convex_cond_std_err}, with all signs flipped.

Finally, we note that both extensions of Lemma \ref{thm:convex_cond_std_err} and \ref{thm:concave_cond_std_err} hold for the entropy-regularized case. The only differences required to prove the results are showing that Lemma \ref{lem:bounded_q_diff} and $|V_f(s)-\overbar{V_f}(s)|\leq L\varepsilon$ hold in entropy-regularized RL. Both statements are trivial given that the necessary soft-max operation is $1$-Lipschitz. Similar results can be derived for the case of compositions, when each subtasks' $Q$-function is replaced by an $\varepsilon$-optimal estimate thereof.

\section{Results Applying to Both Entropy-Regularized and Standard RL}

As we have discussed in the main text; an agent with a large library of accessible functions will be able to transform and compose their primitive knowledge in a wider variety of ways. Therefore, we would like to extend $\mathcal{F}$ to encompass as many functions as possible. Below, we will show that the functions $f\in \mathcal{F}$ characterizing the Transfer MDP Library have two closure properties (additivity and function composition) which enables more accessible transfer functions.

First, let $\mathcal{F}^+$ denote the set of functions $f \in \mathcal{F}$ obeying the convex conditions, and similarly let $\mathcal{F}^-$ denote the set of functions obeying the concave conditions. 

In standard RL, we have the following closure property for addition of functions.
\begin{theorem}
Let $f,g \in  \mathcal{F}^+$. Then $f+g \in \mathcal{F}^+$. Similarly, if $f,g \in \mathcal{F}^-$, then $f+g \in \mathcal{F}^-$.
\end{theorem}
\begin{proof}
    Let $f,g \in \mathcal{F}^+$. 
    
    Convexity:
    The sum of two convex functions is convex.
    
    Subadditive:
    $(f+g)(x+y) = f(x+y)+g(x+y)\leq f(x)+g(x)+f(y)+g(y)=(f+g)(x)+(f+g)(y)$.
    
    Submultiplicative:
    $(f+g)(\gamma x) = f(\gamma x) + g(\gamma x) \le \gamma f(x) + \gamma g(x) = \gamma(f+g)(x)$. 
    
    The proof for $f,g \in \mathcal{F}^-$ is the same with all signs flipped, except for the additional final condition:
    $(f+g)(\max_i x_i) = f(\max_i x_i) + g(\max_i x_i) = \max f(x) + \max g(x) \ge \max f(x) + g(x).$ Although this is not equality as shown in the main text, the condition still suffices. For the case of a single function (no addition, as seen in main text), it can never be the cases that $\max_i f(x_i)~>~\max f(x)$ and therefore was excluded. (Just as $\max_i f(x_i) \le \max f(x)$ is automatically satisfied for the convex conditions.)
\end{proof}

\begin{theorem}[Function Composition]
\label{thm:compos}
For any reward-mapping functions $f$, $g \in \mathcal{F}^+$ ($\mathcal{F}^-$) with $f$ non-decreasing, the composition of functions $f$ and $g$, $h(x) = f(g(x)) \in \mathcal{F}^+ (\mathcal{F}^-)$.
\end{theorem}
\begin{proof}
    Let $f,g \in \mathcal{F}^+$ assume $f: B \to C$ and $g:A \to B$, and let $f$ be non-decreasing. This guarantees that $f(g(x))$ is convex. 
    Additionally, $f(g(x+y)) \leq f(g(x)+g(y)) \leq f(g(x)) + f(g(y))$ by the sublinearity of $g,f$ respectively. Similarly $f(g(\gamma x)) \leq f(\gamma g(x)) \leq \gamma f(g(x))$. 
    
    For the standard RL (concave) condition, note that  for all functions $\mathcal{Q}:~\s~\times~\A \to X$:
    \begin{equation}
        f\left( g\left( \max_{a} \mathcal{Q}(s,a)\right) \right) \geq f\left(\max_{a}~g\left( \mathcal{Q}(s,a)\right) \right) \geq \max_{a}~f\left(g\left( \mathcal{Q}(s,a)\right) \right)
    \end{equation}
    
    For the entropy-regularized condition, we first apply the condition to $g$:

    \begin{equation}
        f\left( g\left( \frac{1}{\beta}\log\mathbb{E}_{a \sim{} \pi_0(a|s)} \exp(\beta \Q(s,a))\right) \right) \le f\left( \frac{1}{\beta}\log\mathbb{E}_{a \sim{} \pi_0(a|s)} \exp(\beta g(\Q(s,a)) )\right) 
    \end{equation}
    Then to $f$:
    
    \begin{equation}
        f\left( g\left( \frac{1}{\beta}\log\mathbb{E}_{a \sim{} \pi_0(a|s)} \exp(\beta \Q(s,a))\right) \right) \le  \frac{1}{\beta}\log\mathbb{E}_{a \sim{} \pi_0(a|s)} \exp\left(\beta f\left( g\left(\Q(s,a)\right) \right) \right)
    \end{equation}
    The reversed statement, when $f,g \in \mathcal{F}^-$ with $f$ non-decreasing has a similar proof and is omitted.
\end{proof}
    
With this result established, we are now able to concatenate multiple transformations. This allows for multiple gates in Boolean logic statements, for example. As stated in the main text, this ability to compose multiple functions will greatly expand the number of tasks in the Transfer MDP Library which the agent may (approximately) solve.

\nocite{openAI}
\begin{@fileswfalse}
\bibliography{uai2023-template}
\end{@fileswfalse}

\end{document}